\documentclass{article}
\usepackage{amsthm}

\usepackage{arxiv}

\usepackage[utf8]{inputenc} 
\usepackage[T1]{fontenc}    
\usepackage{hyperref}       
\usepackage{url}            
\usepackage{booktabs}       
\usepackage{amsfonts}       
\usepackage{nicefrac}       
\usepackage{microtype}      
\usepackage{lipsum}		
\usepackage{graphicx}
\usepackage{natbib}
\usepackage{doi}
\usepackage{algorithm}
\usepackage{algpseudocode}
\usepackage{enumitem}
\usepackage{subfigure}
\usepackage{dsfont}
\usepackage{multirow} 
\usepackage{amsthm}
\usepackage{amsmath}

\usepackage{cleveref}

\usepackage{subcaption}
\usepackage{float} 

\usepackage{siunitx}
\newtheorem{proposition}{Proposition}
\newtheorem{theorem}{Theorem}

\newtheorem{example}{Example}

\newenvironment{pf}{\begin{proof}}{\end{proof}}
\DeclareMathOperator*{\argmin}{arg\,min}

\title{Neural network surrogates with uncertainty quantification for inverse problems in partial differential equations}

\date{} 					

\author{  {Christian Jimenez-Beltran} \\
	School of Mathematics and \\
    Maxwell Institute\\ for Mathematical Sciences\\
        University of Edinburgh\\
	\texttt{s2113174@ed.ac.uk} \\
        \And
        {Aretha L. Teckentrup} \\
	School of Mathematics and \\
    Maxwell Institute\\ for Mathematical Sciences\\
        University of Edinburgh\\
	\texttt{a.teckentrup@ed.ac.uk} \\
        \AND 
	{Antonio Vergari} \\
	School of Informatics\\
        University of Edinburgh\\
	\texttt{avergari@ed.ac.uk} \\
        \And
        {Konstantinos C. Zygalakis}\\
	School of Mathematics and \\
    Maxwell Institute\\ for Mathematical Sciences\\
    University of Edinburgh\\
	\texttt{k.zygalakis@ed.ac.uk} \\
}

\renewcommand{\shorttitle}{\textit{arXiv} Template}

\hypersetup{
pdftitle={A template for the arxiv style},
pdfsubject={q-bio.NC, q-bio.QM},
pdfauthor={David S.~Hippocampus, Elias D.~Striatum},
pdfkeywords={First keyword, Second keyword, More},
}

\begin{document}

\maketitle

\begin{abstract}
Inverse problems for differential equations arise throughout science and engineering, where one seeks to infer unknown model parameters from noisy or incomplete observations. Traditional numerical methods for these problems are often computationally expensive, particularly in Bayesian settings where evaluating the likelihood becomes costly for complex forward models and high-dimensional parameter spaces. To address this challenge, we introduce DeepGaLA, a neural-network surrogate for differential equation solvers that provides uncertainty-aware predictions, reducing overconfident inference when training data are limited. To evaluate the fidelity of the surrogate-induced posterior approximations in practice, we show that a short run of delayed-acceptance Markov chain Monte Carlo can serve as an effective diagnostic. Across a range of numerical experiments, DeepGaLA delivers forward-model approximations with accuracy comparable to established Gaussian-process surrogates, while better maintaining efficiency as parameter dimension grows. Moreover, it can incorporate differential-equation constraints, including in nonlinear settings. Overall, these results indicate that uncertainty-quantified neural surrogates can enable scalable and reliable Bayesian inference for inverse problems in complex systems.
\end{abstract}


\section{Introduction}

Differential equations (DEs) are fundamental mathematical tools for modeling and analyzing a wide range of phenomena, from biological systems to fluid dynamics. However, modeling real-world phenomena such as weather dynamics or neutron transport in cancer therapy often involves complex, high-dimensional DEs, which require estimating parameters that are not directly observable. To address this challenge, recent approaches employ data-driven frameworks to infer model parameters from observations. In this context, inverse problems play a central role, as they provide a framework to infer parameters from incomplete, indirect, and noisy observations \cite{engl1996regularization,kaipio2005statistical}.  Inverse problems are, however, often ill-posed, meaning that they may admit multiple solutions or that the solution may be highly sensitive to the observational data \cite{engl1996regularization,kaltenbacher2009iterative,Benning_Burger_2018}. A Bayesian framework addresses these challenges by regularizing the problem and incorporating prior information as well as parameter uncertainty. In this setting, the goal is to obtain the posterior distribution, which characterizes the statistical properties of the parameters  conditioned on the available observations \cite{kaipio2005statistical,KAIPIO2007493,stuart_inverse,latz2020}.

Although the Bayesian framework offers significant advantages including uncertainty quantification, the posterior distribution is rarely available in closed form. Consequently, the posterior distribution is typically approximated using either variational or Markov chain Monte Carlo (MCMC) methods. Variational methods involve defining a parameterized family of candidate distributions and selecting the member that best approximates the posterior \cite{jordan1999introduction,blei2017variational}. In contrast, MCMC methods approximate the posterior by generating samples from it \cite{metropolis,hastings,robert1999monte}, which can then be used to estimate expectations and other quantities of interest. However, a key challenge of MCMC methods is that they typically require evaluating the likelihood function at each iteration, often millions of times to ensure the convergence of the MCMC method. Since each likelihood evaluation involves solving the DE numerically, the overall computation can become extremely expensive. 

Approaches such as multilevel (Markov chain) Monte Carlo \cite{mlmc,aretha_mmcmc} and Delayed Acceptance MCMC (DA-MCMC) \cite{delayed_mcmc,multilevel_da} have been developed to reduce the computational cost of posterior sampling. Other approaches such as surrogate models have been studied to enable an efficient evaluation of the likelihood function. Examples of such models include polynomial chaos expansions \cite{XIU2003137,MARZOUK2007560}, sparse grid methods \cite{nobile_sparse,marzouk2009stochastic}, and Gaussian processes \cite{kennedy2001bayesian,sacks1989design,Bai_2024,GP_nonlpde,GP_nonlinear_error}.  More recently, neural networks have proven to be highly effective models for both regression and classification tasks, including image recognition and language processing, among others \cite{ian_deepl,deep_vision,deep_reinforcement}. Building on these advances and the universal approximation capability of neural networks \cite{universal}, deep learning approaches have been studied to introduce a new class of surrogate models that use neural networks to approximate solutions to differential equations, allowing fast and scalable solutions for high-dimensional or computationally expensive problems. 

Deep learning approaches can be broadly divided into three categories. One class consists of data-driven models, such as DeepONet \cite{deeponet}, Fourier Neural Operators (FNO) \cite{fourieroperator}, and Laplace Neural Operators \cite{laplacenop}, which learn mappings between function spaces. These methods have demonstrated strong performance in approximating solutions of parametric PDEs; however, they typically require large amounts of training data, which are often generated using expensive numerical solvers. For comprehensive reviews of neural operator methods, we refer to \cite{noperator_analysis,neuraloperator,cost_accuracy}.

Physics-informed deep models constitute another class of approaches, as they incorporate physical laws directly into the neural network and may not require the explicit solution of the PDE for training. One of the first models introducing this concept is the Deep Galerkin Method (DGM), proposed in \cite{dgm} for approximating solutions of high-dimensional Hamilton–Jacobi–Bellman PDEs and other quasilinear parabolic PDEs. DGM employs a neural network architecture similar to Long Short-Term Memory (LSTM) networks \cite{lstm} and is trained using randomly sampled collocation points in the domain of the PDE to approximate solutions efficiently. Other examples include the popular Physics-Informed Neural Networks (PINNs) \cite{pinns}, which embed the governing physical laws into the loss function, and the Deep Ritz Method (DRM) \cite{deepritz}, which reformulates the PDE as a variational problem to construct the loss function. Compared to DGM, PINNs were originally introduced to solve non-parametric PDEs; however, there has been substantial research extending PINNs to parametric problems and addressing various challenges, including training stability and generalization. For a comprehensive overview, we refer the reader to recent surveys on Physics-Informed Neural Networks, which discuss both their theoretical foundations and practical limitations \cite{pinns_review,DeRyck_Mishra_2024}.

A third class of deep learning approaches for solving PDEs combines neural operators with physics-informed losses. Examples include the Physics‑Informed Neural Operator (PINO) \cite{pino}, which integrates a FNO with a physics-based loss, and the Physics‑Informed DeepONet \cite{pideeponet}, which embeds PDE constraints directly into the DeepONet training process. These methods reduce the dependence on large training datasets while ensuring that the learned solution operators satisfy the underlying physical laws.

Despite the promise of these approaches, training neural networks to achieve a desired level of accuracy remains challenging, as it often requires large amounts of data, and the choice of architecture can significantly affect performance. Moreover, neural networks tend to be overconfident in their predictions \cite{kristiadi2020being}, even when their outputs are far from the ground truth, making them unreliable as surrogate models, particularly at unseen collocation points (i.e., out-of-distribution scenarios). Therefore, to effectively use deep neural networks as surrogates, it is crucial to obtain calibrated uncertainty estimates for the numerical approximation of DEs \cite{hennig2015probabilistic,conrad2017statistical}.

To address these limitations, we propose DeepGaLA, Deep Galerkin via Laplace Approximation, a neural network method for approximating solutions of parametric PDEs while providing uncertainty quantification. Our approach combines the DGM with a Bayesian neural network obtained by randomizing the last layer and applying a Laplace approximation \cite{laplace}.  To the best of our knowledge, this is the first method to integrate Laplace-based uncertainty quantification into the DGM framework for parametric PDEs. The resulting model not only approximates the forward model efficiently but also provides calibrated uncertainty estimates, which are particularly important in limited-data regimes and for Bayesian inverse problems. The main contributions of this work are as follows:

\begin{enumerate}
    \item We provide a natural Bayesian formulation for the deep Galerkin method by reinterpreting the minima achieved by the loss function as the maximum a posteriori (MAP) estimate associated with a Gaussian likelihood function. In our proposed method, DeepGaLA, the corresponding posterior is then approximated using a Laplace approximation, enabling efficient uncertainty estimation for neural network surrogates for partial differential equations. 

    \item Through a series of numerical experiments, we demonstrate (i) the computational efficiency of DeepGaLA and (ii) the importance of  incorporating uncertainty in the surrogate model, particularly when the neural network is trained with limited data. We further show that DeepGaLA achieves comparable performance to popular alternative surrogates based on Gaussian processes, albeit with an improved scaling to high-dimensional parameter spaces and applicability also to non-linear PDEs.

    \item We propose to use a short run of the delayed acceptance MCMC algorithm to evaluate the accuracy of surrogate posteriors when having access to a reference "true" posterior is infeasible, and provide a theoretical analysis to justify this approach.

\end{enumerate}

The rest of the paper is organized as follows. In Section \ref{sec:2}, we introduce the PDE formulation and the inverse problems of interest. In Section \ref{sec:3}, we discuss surrogate models and the different posterior approximations that arise from using random surrogate models. We present the methodology for constructing a neural random surrogate and also review Gaussian Processes and their use as surrogates. Section \ref{sec:4} focuses on MCMC methods used to sample from and assess the quality of posterior approximations.  In Section \ref{sec:5}, we present a number of numerical experiments demonstrating the computational benefits of surrogates, with particular emphasis on neural network surrogates. Finally, Section \ref{sec:6} concludes the paper by summarizing the proposed methodology, discussing the main findings, and outlining potential directions for future work.

\section{Preliminaries}\label{sec:2}

In this section, we set up the differential equations formulation and provide examples of the problems of interest. We also review the Bayesian framework for inverse problems and discuss some of the advantages and disadvantages of this setting.

\subsection{Parametric PDE Formulation}
Consider a bounded domain $D \subset \mathbb{R}^{d_{x}}$ with boundary $\partial D$ and a bounded parameter space $\mathcal{T} \subset \mathbb{R}^{d_{\theta}}$. Let $\mathcal{W}$ and $\mathcal{Y}$ denote Hilbert spaces on $D$, and let $\mathcal{W}_{\partial}$ and $\mathcal{Y}_{\partial}$ denote Hilbert spaces on $\partial D$. A general parametric PDE can then be written as:
\begin{equation}\label{eq:pde}
\begin{split}
    \mathcal{A}_{\theta_{\mathcal{A}}}\left[ u(x,\theta)\right] &= f(x, \theta_{\mathcal{A}}),  \quad x \in D,\quad \theta \in \mathcal{T},\\
    \mathcal{B}_{\theta_{\mathcal{B}}}\left[ u(x,\theta)\right] &= g(x, \theta_{\mathcal{B}}), \quad x \in \partial D, \quad \theta \in \mathcal{T},
\end{split}
\end{equation}
where $u(\cdot,\theta) \in \mathcal{W} \cap \mathcal{W}_{\partial}$ is the solution of the PDE for a given $\theta=(\theta_{\mathcal{A}}, \theta_{\mathcal{B}}) \in \mathcal{T}$, $\mathcal{A}_{\theta_{\mathcal{A}}}: \mathcal{W} \times \mathcal{T} \rightarrow \mathcal{Y}$ is a general differential operator parameterized by $\theta_{\mathcal{A}}$, $\mathcal{B}_{\theta_{\mathcal{B}}}: \mathcal{W}_{\partial} \times \mathcal{T} \rightarrow \mathcal{Y}_{\partial}$ denotes the initial and boundary conditions operator parameterized by $\theta_{\mathcal{B}}$, and $f \in \mathcal{Y}$, $g \in \mathcal{Y}_{\partial}$ are the forcing and boundary data parametrized by $\theta_{\mathcal{A}}$ and $\theta_{\mathcal{B}}$, respectively. Note that $\mathcal{A}$ may be nonlinear, and the input variable $x$ could include time as a component.

\begin{example}
Consider the Darcy flow equation, which is a prototypical second order linear elliptic PDE used to model fluid flow in porous media,
\begin{equation*}\label{eq:pde_example1}
\begin{split}
    -\nabla \cdot \big({\alpha(x,\theta)}\,\nabla u(x,\theta)\big) &= f(x),  
    \quad x \in D,\quad \theta \in \mathcal{T},\\
    u(x,\theta) &= 0, \quad x \in \partial D,
\end{split}
\end{equation*}
with $u(\cdot,\theta) \in H_{0}^{1}(D)$ denoting the pressure field and 
$\alpha(\cdot,\theta) \in L^{\infty}(D)$ the permeability.  The coefficient field $\alpha(x,\theta)$ is typically modeled as a random field and parametrized using a 
Karhunen--Lo\`eve (KL) expansion, 
\begin{equation*}
\log \alpha(x,\theta) = \bar{\alpha}(x) + \sum_{i=1}^\infty \sqrt{\mu_i}\,\phi_i(x)\,\theta_{i},
\quad \theta_{i} \sim \mathcal{N}(0,1),
\end{equation*}
where $\{\mu_i, \phi_i(x)\}_{i=1}^\infty$ are the eigenvalues and eigenfunctions of the covariance operator of the random field. In practice, this infinite sum is truncated to the first $d_{\theta}$ terms. For this example, the operators are 
\begin{equation*}
\mathcal{A}_{\theta_{\mathcal{A}}}[\,\cdot\,] = -\nabla \cdot \big(\alpha(x,\theta)\,\nabla(\cdot)\big) \quad \text{and}\quad \mathcal{B}_{\theta_{\mathcal{B}}}[\,\cdot\,] = I.
\end{equation*}
\end{example}

\begin{example} 
As a prototypical example of a nonlinear PDE, we consider the 2D incompressible Navier–Stokes equations in vorticity–stream function form, which describe viscous incompressible flow:
\begin{equation*}
\begin{split} 
 \frac{\partial w(z,t,\theta)}{\partial t} + U(z,t,\theta) \cdot \nabla w(z,t,\theta) &= \nu \nabla^2 w(z,t,\theta) , \quad (z,t) \in D, \quad \theta \in \mathcal{T},\\
 \nabla^2 \psi(z,t,\theta) &= -w(z,t,\theta), \quad (z,t) \in D, \\
 w(z, 0,\theta) &= w_0(z,\theta), \quad z \in D_{s},
\end{split} 
\end{equation*}
where $D = D_{s} \times (0,T)   \subset \mathbb{R}^{d_{3}}$, $D_{s} \subset \mathbb{R}^{2}$ denotes the spatial domain 
and $T \in \mathbb{R}^{+}$. The term $U(z,t,\theta) = (u_1(z,t,\theta), u_2(z,t,\theta))$ is the velocity, with components derived from the stream function as $u_1(z,t,\theta) = \partial \psi(z,t,\theta) /\partial z_{2}$ and $ u_2 = -\partial \psi(z,t,\theta)/\partial z_{1}$. We additionally enforce periodic boundary conditions $w(\cdot,\theta) \in L^{2}\!\big((0,T); H^{1}_{per}(D_{s})\big)$, and the initial vorticity is modeled through a truncated KL expansion, 
\begin{equation*}
w_{0}(z,\theta) = \bar{w}(z) + \sum_{i=1}^{d_{\theta}} \sqrt{\mu_i}\,\phi_i(z)\,\theta_i, \quad\theta_i \sim \mathcal{N}(0,1),
\end{equation*}
where $\{\mu_i,\phi_i(z)\}_{i=1}^{d_{\theta}}$ are again the eigenvalues and eigenfunction of the covariance operator of the initial conditions. With $x:=(z,t)$, $u(x,\theta):= (w(x,\theta), \psi(x,\theta))$ and $f=g=0$, we then have
\begin{align*} 
 \mathcal A_{\theta_{\mathcal{A}}} [u] &= \begin{bmatrix}
    \mathcal A_{\theta_{\mathcal{A}}}^1[u] \\ 
    \mathcal A_{\theta_{\mathcal{A}}}^2[u] 
 \end{bmatrix}, \quad \textrm{with} \quad 
 \mathcal{A}_{\theta_{\mathcal{A}}}^1[\,u\,]= \partial_tw + U(x,\theta) \cdot \nabla w  - \nu \nabla^2 w, \quad \mathcal{A}_{\theta_{\mathcal{A}}}^2[\,u\,] =\nabla^2 \psi +w,   \\
\mathcal{B}_{\theta_{\mathcal{B}}}[u] &=
\begin{bmatrix} 
\mathcal{I}_{\theta_{\mathcal{B}}}[u]\\
\mathcal{B}_{p,\theta_{\mathcal{B}}}[u]
\end{bmatrix}, \quad \textrm{with} \quad 
\mathcal{I}_{\theta_{\mathcal{B}}}[u] = w(x,\theta)\big|_{t=0} , \quad
\mathcal{B}_{p,\theta_{\mathcal{B}}}[u] \rightarrow  
w(\cdot,t,\theta), \psi(\cdot,t,\theta)  \text{ are periodic on } D_{s}.
\end{align*}
\end{example}

\subsection{Bayesian Inverse Problems}
The objective of inverse problems in differential equations, is to determine the parameters $\theta \in  \mathcal{T} $ given observations $y \in \mathbb{R}^{d_{y}}$ of the model solution subject to noise $\eta$. Typically, one assumes that the data are modeled in the following way:  
\begin{equation}
\label{inv_eq}
y = \mathcal{G}_{\mathrm{X}}(\theta) + \eta,
\end{equation}  
where $\eta \sim \mathcal{N}(0,\Gamma)$, $\Gamma \in \mathbb{R}^{d_{y} \times d_{y}}$, 
and $\mathcal{G}_{\mathrm{X}}: \mathcal{T} \rightarrow \mathbb{R}^{d_{y}}$ is known as the observation operator. For instance, this operator can be defined as $\mathcal{G}_{\mathrm{X}}(\theta) = \{u(x_{i},\theta)\}_{i=1}^{dy}$, mapping the parameter $\theta$ to the solution of the PDE evaluated at the spatial (and possibly temporal) observation points $X=\{x_{1},..,x_{d_{y}}\} \subset \overline{D}$. 

In a Bayesian framework, $\theta$ and $y$ are treated as random variables. Then, the goal is to determine the probability measure of $\theta$ conditioned on $y$, denoted by $\mu^y$, with a density $\pi^{y}(\theta)$. In the absence of data, $\theta$ is distributed according to the prior measure $\mu_0$, which we assume admits a density $\pi_0(\theta)$ with respect to the Lebesgue measure. Then, by Bayes’ formula, we have
\begin{equation} 
\label{bayes}
\pi^{y}(\theta) \propto l(y|\theta)\pi_0(\theta), 
\end{equation}
where we identify $l(y|\theta) \propto  \exp\left(-\frac{1}{2} \| y -\mathcal{G}_{\mathrm{X}}(\theta) \|^{2}_{\Gamma}\right)$ as the data likelihood  describing the probability of obtaining the observations $y$ given a set of parameters $\theta$ and  $\|z\|^{2}_{\Gamma}=z^{T}\Gamma^{-1}z$ is the norm weighted by $\Gamma^{-1}$.
The Bayesian framework provides a principled form of regularization via the prior distribution and under suitable regularity conditions the Bayesian inverse problem is well-posed, admitting a unique posterior measure that depends continuously on the data in the Hellinger distance (see e.g. \cite{stuart_inverse,latz2020}).

A major challenge of the Bayesian framework is that one often relies on Markov Chain Monte Carlo (MCMC) methods to compute quantities of interest, such as expectations, due to the lack of a closed-form expression for $\pi ^{y}(\theta)$. Typically, MCMC algorithms require millions of iterations to obtain a good approximation of the posterior distribution, which in turn requires evaluating the likelihood millions of times. Since each likelihood evaluation involves numerically solving a PDE, the overall computational cost can become prohibitively high, particularly for high-dimensional problems.

\section{Surrogate Models In Bayesian Inverse Problems}\label{sec:3}

This section introduces surrogate models as a means of reducing the high computational cost associated with evaluating the likelihood in the Bayesian inverse problem setting. We begin by defining surrogate models and discussing how they can be employed to approximate the posterior distribution. A central part of this section focuses on neural network–based surrogates \cite{ian_deepl,universal, pinn_error,dgm_tony}: we first present the fundamentals of neural networks and explain how they can be applied to approximate solutions of parametric PDEs. We then introduce a framework for constructing neural network surrogates equipped with uncertainty quantification. The section concludes with a review of Gaussian processes, a class of models that inherently provide uncertainty estimates alongside their predictions and have been shown to be effective surrogates in a variety of settings. These will be used as a benchmark for comparison in our numerical examples.

\subsection{Surrogate models and approximate posteriors}
Surrogates, or emulators, are models designed to accelerate the evaluation of computationally expensive models; in our case, they are used to accelerate the forward model $\mathcal{G}_{\mathrm{X}}$. Instead of repeatedly solving the full numerical model, surrogate approaches approximate the input–output map of the forward model. Once constructed, the surrogate can be evaluated at a significantly lower computational cost, making Bayesian inference feasible. Conventional surrogates typically provide only deterministic predictions, whereas random surrogates also quantify uncertainty, which is essential to assess the reliability of the model. A common example is Gaussian Process (GP) regression, a random surrogate that has been extensively studied in terms of performance and the error induced by its use \cite{stein1999interpolation,kennedy2001bayesian,sacks1989design, aretha_posterior, kostas_aretha}.  

Suppose we employ a random surrogate, characterized by a discretization parameter or training parameter $\mathrm{N}$, and let $\mathcal{G}_{\mathrm{X}}^{\mathrm{S,N}}(\theta) \sim \mu^{\mathrm{S,N}}$ denote the corresponding approximation to the observation operator of a surrogate $S$ distributed according to a measure $\mu^{\mathrm{S,N}}$, where $X$ is fixed. The work of \cite{Lie_2018, aretha_posterior} introduces two approaches to approximate the posterior distribution  \eqref{bayes} when employing a random surrogate: the \textit{mean} and \textit{marginal} approximations. The \textit{mean-based approximation} of the posterior distribution is obtained using only the expected value of the random surrogate $\mathbb{E}_{\mu^{\mathrm{S,N}}}[\mathcal{G}_{\mathrm{X}}^{\mathrm{S,N}}(\theta)]$. This yields the following expression:
\begin{equation}
\label{eq:mean-approx}
\begin{split}
    \pi^{y,\mathcal{G}_{\mathrm{X}}^{\mathrm{S,N}}}_{\mathrm{mean}}(\theta) &= \frac{1}{Z_\mathrm{mean}^{y,\mathcal{G}_{\mathrm{X}}^{\mathrm{S,N}}}} \exp\left(-\frac{1}{2} \| y -\mathbb{E}_{\mu^{\mathrm{S,N}}}[\mathcal{G}_{\mathrm{X}}^{\mathrm{S,N}}(\theta)] \|^{2}_{\Gamma}\right)\pi_{0}(\theta),\\
    Z_\mathrm{mean}^{y,\mathcal{G}_{\mathrm{X}}^{\mathrm{S,N}}}& = \mathbb{E}_{\mu_{0}}\left[ \exp\left(-\frac{1}{2} \| y -\mathbb{E}_{\mu^{\mathrm{S,N}}}[\mathcal{G}_{\mathrm{X}}^{\mathrm{S,N}}(\theta)] \|^{2}_{\Gamma}\right) \right].
\end{split}
\end{equation}
The \textit{marginal approximation}, in contrast, directly uses the random surrogate in the likelihood and then takes the expectation of the resulting expression. This leads to:
\begin{equation}
\label{eq:marginal-approx}
\begin{split}
 \pi^{y,\mathcal{G}_{\mathrm{X}}^{\mathrm{S,N}}}_{\mathrm{marginal}}(\theta) &= \frac{1}{Z_\mathrm{marginal}^{y,\mathcal{G}_{\mathrm{X}}^{\mathrm{S,N}}}} \mathbb{E}_{\mu^{\mathrm{S,N}}} \left[ \exp\left(-\frac{1}{2} \| y -\mathcal{G}_{\mathrm{X}}^{\mathrm{S,N}}(\theta) \|^{2}_{\Gamma}\right) \right]\pi_{0}(\theta), \\
 Z_\mathrm{marginal}^{y,\mathcal{G}_{\mathrm{X}}^{\mathrm{S,N}}} &= \mathbb{E}_{\mu_{0}} \left[ \mathbb{E}_{\mu^{\mathrm{S,N}}} \left[ \exp\left(-\frac{1}{2} \| y -\mathcal{G}_{\mathrm{X}}^{\mathrm{S,N}}(\theta) \|^{2}_{\Gamma}\right)\right] \right].   
\end{split}
\end{equation}
Surrogates are often constructed from a limited computational resources, which can introduce uncertainty or limit their accuracy. Taking into account this uncertainty is therefore essential. In this context, the marginal approximation offers a key advantage: it incorporates the uncertainty of the surrogate model, thereby improving robustness in scenarios where the emulator may yield inaccurate approximations \cite{conrad2017statistical,Bai_2024}.

In the particular case where the random surrogate is Gaussian, $\mathcal{G}_{\mathrm{X}}^{\mathrm{S,N}}(\theta) \sim \mathcal{N}(m_{\mathrm{S,N}}^{\mathcal{G}}(\theta), \Gamma_{\mathrm{S,N}}(\theta,\theta'))$, with $m_{\mathrm{S,N}}^{\mathcal{G}}(\theta) \in \mathbb{R}^{d_{y}}$ and $\Gamma_{\mathrm{S,N}}(\theta,\theta') \in \mathbb{R}^{d_{y} \times d_{y}}$, the equations \eqref{eq:mean-approx} and \eqref{eq:marginal-approx} can be  simplified to:
\begin{equation}
\label{eq:mean-approx_gauss}
\pi^{y,\mathcal{G}_{\mathrm{X}}^{\mathrm{S,N}}}_{\mathrm{mean}}(\theta)  \propto  \exp\left(-\frac{1}{2} \| y - m_{\mathrm{S,N}}^{\mathcal{G}}(\theta) \|^{2}_{\Gamma}\right)\pi_{0}(\theta),
\end{equation}
\begin{equation}
\label{eq:marginal-approx_gauss}
\pi^{y,\mathcal{G}_{\mathrm{X}}^{\mathrm{S,N}}}_{\mathrm{marginal}}(\theta) \propto \frac{1}{\sqrt{\det(\Gamma + \Gamma_{\mathrm{S,N}}(\theta,\theta))}}  
\exp\left(-\frac{1}{2} \| y - m_{\mathrm{S,N}}^{\mathcal{G}}(\theta) \|^{2}_{\Gamma + \Gamma_{\mathrm{S,N}}(\theta,\theta)}\right)\pi_{0}(\theta).
\end{equation}
From these expressions, we observe that the marginal approximation accounts for the surrogate’s uncertainty. As the surrogate becomes more confident in its predictions, i.e., as $\Gamma_{\mathrm{S,N}}$ decreases, the marginal approximation approaches the mean approximation, particularly when $\Gamma_{\mathrm{S,N}} \ll \Gamma$ \cite{Bai_2024}. Furthermore, both the mean and marginal approximations under suitable assumptions converge to the true posterior as $N \rightarrow \infty$ \cite{aretha_posterior,kostas_aretha}.

\subsection{Neural Networks}
 
We are interested in building a neural network surrogate that approximates the observation operator, $\mathcal{G}_X$. First, we define a neural network following the presentation of \cite{modern_dp,approximation_dnn}. A feedforward neural network is a function $f_{a}(\cdot;\mathrm{W}): \mathbb{R}^{d_{0}} \rightarrow \mathbb{R}^{d_{L}}$
characterized by its architecture $a$  and its weights or parameters $\mathrm{W}$. The architecture $a = (N_{n}, \sigma)$ 
is characterized by the number of neurons per layer $N_{n} = (d_{0}, \dots, d_{L}) \in \mathbb{N}^{L+1},$
where $L$ is the number of layers, and by an activation function $\sigma: \mathbb{R} \rightarrow \mathbb{R}$. Let $C_{k}(z) = \mathrm{W}_{k} z + b_{k}$ denote a linear transformation, where $\mathrm{W}_{k} \in \mathbb{R}^{d_{k} \times d_{k-1}},
b_{k} \in \mathbb{R}^{d_{k}}$. Denote the operation $\sigma( C_k(x))= \big(\sigma((C_k)_1(x)), \ldots, \sigma((C_k)_{d_{k}}(x))\big)$ as the application of the activation function component-wise and $\Phi^{(0)}(x;\mathrm{W})=x$ the network input. Then a feedforward neural network with architecture $a$ and weights 
$\mathrm{W} = \{(\mathrm{W}_{i}, b_{i})\}_{i=1}^{L}
$ has the following composition rule:
\begin{equation}  \label{eq:nn_composition}
\begin{split}
f_a(x;\mathrm{W})& = C_{L}(\Phi^{(L-1)}(x;\mathrm{W})),\\
\Phi^{(l+1)}(x;\mathrm{W}) &= \sigma\big(C_{\ell+1}(\Phi^{(l)}(x;\mathrm{W}))\big),
\quad l = 0, \dots, L-2.  
\end{split}
\end{equation}
The input layer for the neural network corresponds to $\Phi^{(1)}$, while the hidden layers are represented by the intermediate functions $\Phi^{(l+1)}$ for $l = 1,\dots,L-2$. . The most commonly used activation functions are the hyperbolic tangent and the rectified linear unit (ReLU), defined as $\sigma(x) = \tanh(x)$ and $\sigma(x) = \max(0,x)$, respectively \cite{ian_deepl, pinns_review}. 
 
Given  a training dataset $\mathcal{D} =\{(x_i, y_i)\}_{i=1}^{N}$, where $x_i \in \mathbb{R}^{d_0}$ and $y_i \in \mathbb{R}^{d_L}$, the goal of deep learning is to find the optimal weights $\mathrm{W}^*$ of a neural network $f_{a}(\cdot;\mathrm{W})$ such that the following empirical loss (or empirical risk) is minimized: 
\begin{equation*}
     \mathrm{W}^* = \argmin_{\mathrm{W}} \mathcal{L}(\mathcal{D};\mathrm{W}) = 
     \argmin_{\mathrm{W}}  \frac{1}{N}\sum_{i=1}^{N} \ell\left(f_{a}(x_i;\mathrm{W}),y_i \right).
\end{equation*}
Here $\ell : \mathbb{R}^{d_L} \times \mathbb{R}^{d_L} \rightarrow \mathbb{R}$ denotes the loss function, which measures the error error between the prediction $f_{a}(x_i;\mathrm{W})$ and the true value $y_{i}$. 
Typical choices include the mean squared error (MSE) for regression, $\ell\big(f_{a}(x_i;\mathrm{W}), y_i\big) 
= \left\|f_{a}(x_i;\mathrm{W}) - y_i \right\|_2^2,$
and the cross-entropy loss for classification tasks \cite{ian_deepl}. 

The minimization problem is typically solved using optimization techniques such as gradient-based methods. Once the minimization has been carried out, $f_{\mathrm{W}^{*}}^{a}$ has been trained. Despite the success of neural networks, it is still unclear how to design a neural network that guarantees achieving a prescribed error level. Several works have focused on improving our understanding of neural networks and on establishing approximation error bounds. Some works in this direction include \cite{tanhnnapprox}, which studied the approximation error of two-layer neural networks with hyperbolic tangent activation functions, \cite{relu_fem}, which investigated the connection between ReLU activation functions and the finite element method (FEM), and \cite{sharp_nn}, in which theoretical results concerning the approximation error bounds of shallow networks were established.

In particular, the work of \cite{rates_nn_jakob} presents rates of convergence for ReLU neural networks for a specific class of functions. Specifically, they derive these rates by first showing that Hermite polynomials can be approximated by ReLU neural networks, 
and then by showing that functions which can be expanded in terms of Hermite polynomials can also be approximated by ReLU neural networks. This result demonstrates the potential capabilities of neural networks; however, it remains unclear how to construct and train such networks to guarantee a certain level of accuracy. An informal version of Theorem 4.7 of \cite{rates_nn_jakob}, which applies to ReLU activation functions and approximation of functions $f$ of $d_\theta < \infty$ parameters, is stated as follows:

\begin{proposition}
\label{prop:nn_rates}
Let $f: \mathcal T \rightarrow \mathbb{R}$ be sufficiently regular. Then there exist constants $C_1, C_2$, depending on $f$ and $d_{\theta}$, such that for every $M \in \mathbb{N}$ there exists a feedforward neural network $f_{a}(\cdot;\mathrm{W})$ with architecture $a=(N_{n},  \sigma)$ and $\sigma= \max(0,x)$ satisfying  
\begin{equation*}
   \|f- f_{a}(\cdot;\mathrm{W})\|_{L^{2}({\mathcal T})}  \leq C_1 \exp \left(-2^{-\frac{1}{2}}C_2 M^{\frac{1}{2d_{\theta}+7}}\right), 
\end{equation*}
and 
\begin{equation*}
    \sum_{i=1 }^{L}d_{i}d_{i-1} +d_{i}\leq CM\left(1+ \log(M) \right), \quad L\leq CM^{\frac{3}{2d_{\theta}+7}}\left(1+\log(M)\right)^{2}.
\end{equation*}
\end{proposition}

In the following section, we review a method that employs neural networks to approximate solutions of parametric PDEs and demonstrate how this methodology can be used to construct a random surrogate model.

\subsubsection{Neural Parametric PDE Solver}\label{subsubsec:neuralPDE}

The Deep Galerkin Method (DGM) was introduced in \cite{dgm} as a mesh-free approach for approximating solutions to parametric PDEs, enabling efficient scaling to high-dimensional problems in space and parameter variable. The original work demonstrated the method's capabilities on high-dimensional Hamilton–Jacobi–Bellman PDEs and other quasilinear parabolic PDEs, and its effectiveness as a surrogate for solving inverse problems was later shown in \cite{dgm_tony}. Next, we will review this methodology.

Consider the problem given by equation \eqref{eq:pde}, and a neural network 
$f_{a}(\cdot;\mathrm{W}): \bar{D} \times \mathcal{T} \rightarrow \mathcal{W} \cap \mathcal{W}_{\partial}$. The training dataset is $\mathcal{D} = \mathcal{D}_{\mathcal A} \cup \mathcal{D}_{\mathcal B} = \left\{ \mathcal{D}_{\mathcal A,i} \right\}_{i=1}^{N_{\mathcal A}} \cup \left\{ \mathcal{D}_{\mathcal B,i} \right\}_{i=1}^{N_{\mathcal B}}$ with collocation points $\mathcal{D}_{\mathcal A,i} = ( x_{i}, \theta_{i} )$ and 
$\mathcal{D}_{\mathcal B,i} = ( x^{\partial}_{i}, \theta_{i} )$ where 
$x_{i} \sim \pi_{D}$, $x^{\partial}_{i} \sim \pi_{\partial D}$, and 
$\theta_{i} \sim \pi_{\mathcal{T}}$, independently. Here, $\pi_{D}$, 
$\pi_{\partial D}$, and $\pi_{\mathcal{T}}$ are probability densities 
defined on $D$, $\partial D$, and $\mathcal{T}$, respectively. We then define the training loss as:
\begin{equation}\label{eq:loss_pde}
    \begin{split}
    &\mathcal{L}(\mathcal{D};\mathrm{W})=  \lambda_{\mathcal A}\mathcal{L}_{\mathcal A}( \mathcal{D}_{\mathcal A};\mathrm{W})+  \lambda_{\mathcal B}\mathcal{L}_{\mathcal B}(\mathcal{D}_{\mathcal B};\mathrm{W}),\\
     &\mathcal{L}_{\mathcal A}( \mathcal{D}_{\mathcal A};\mathrm{W}) =\frac{1}{N_{\mathcal A}}\sum^{N_{\mathcal A}}_{i = 1}\Big| \mathcal{A}_{\theta_{\mathcal A,i}}\left[f_{a}(\mathcal{D}_{\mathcal A,i};\mathrm{W})\right] -f(x_{i},\theta_{\mathcal A,i})\Big|^{2}, \\
     &\mathcal{L}_{\mathcal B}(\mathcal{D}_{\mathcal B};\mathrm{W}) = \frac{1}{N_{\mathcal B}}\sum^{N_{\mathcal B}}_{i = 1}   \Big|\mathcal{B}_{\theta_{\mathcal B,i}} \left[f_{a}(\mathcal{D}_{\mathcal B,i};\mathrm{W})\right] -g(x^{\partial}_{i},\theta_{\mathcal B,i})\Big|^{2}.
    \end{split}
\end{equation}
where $\mathcal{L}_{\mathcal A}( \mathcal{D}_{\mathcal A};\mathrm{W})$ measures the error in the approximation of the differential operator, and $\mathcal{L}_{\mathcal B}(\mathcal{D}_{\mathcal B};W)$ measures the error in the boundary and/or initial conditions. The terms $\lambda_{\mathcal A}$ and $\lambda_{\mathcal B}$ are weighting factors that adjust the contributions of different loss terms, accounting for the fact that they may be on different scales. 

In the case of time-dependent PDEs, it has been noted the loss defined in \eqref{eq:loss_pde} can lead to violations of temporal causality. Therefore, \cite{temporal_causality} recommends partitioning the temporal domain into $M$ segments and instead defining $\mathcal{L}_{\mathcal A}$ as follows:
\begin{equation*}
    \mathcal{L}_{\mathcal A}( \mathcal{D}_{\mathcal A};\mathrm{W}) =  \frac{1}{M}\sum\nolimits^{M}_{i = 1} w_{i}\mathcal{L}_{\mathcal A}(\mathcal D_{\mathcal A}^i;W),\quad
   w_{i} = \exp\left(-\sum_{l=1}^{i-1} \mathcal{L}_{\mathcal A}(\mathcal D_{\mathcal A}^l; W) \right), \quad \text{for} \quad i=2,\ldots,M. 
\end{equation*}
Here, $\mathcal D_{\mathcal A}^i \subset \mathcal D_{\mathcal A}$ contains all the collocation points in the $i$th time segment. In doing so, the neural network is encouraged to learn the solution of the PDE progressively along the time axis, since $\mathcal{L}_{\mathcal A}^{i}$ will not be minimized unless the preceding losses $\{\mathcal{L}_{\mathcal A}^{k}\}_{k=1}^{i-1}$ are sufficiently small to ensure that $w_{i}$ is large.           

\begin{algorithm}[ht]
\caption{Training Pipeline of Deep Galerkin Method}\label{algo1}
\textbf{Input: } Neural network $f_{a}(\cdot;\mathrm{W})$, data training $\mathcal{D}$, initial weights $\{\lambda_{i}\}_{i=1}^K,   \{w_i\}_{i=1}^M$ set to one, learning rate $\gamma$, weight update frequency $l$, and total iterations $S$.

\textbf{Output: } Trained network $f_{a}(\cdot;\mathrm{W}^*)$ with optimal weights $\mathrm{W}^*$.

\begin{enumerate}
    \item Similar to \eqref{eq:loss_pde}, consider a loss function with $K$ components as $\mathcal{L}(\mathcal{D};\mathrm{W} )= \sum_{i=1}^{K}\lambda_{i} \mathcal{L}(\mathcal{D}_{i};\mathrm{W} )$ 
    
    \item \textbf{for} $n = 1$ \textbf{to} $S$ \textbf{do}
    \begin{enumerate}

        \item Evaluate $\mathcal{L}_i(D_i, \mathrm{W})$ for $i=\{1,\ldots, K \}$.
        
        \item \textbf{if} $n \mod l = 0$ \textbf{then}
        \begin{enumerate}
            \item Compute adaptive loss weights:
            \begin{equation*}
               \lambda_i = \frac{1}{K} \cdot \frac{\sum_{j}^{K} \|\nabla_\mathrm{W} \mathcal{L}_{j}\|_2}{\|\nabla_\mathrm{W} \mathcal{L}_i\|_2}, \quad \text{for }i =\{1,\ldots, K \}. 
            \end{equation*}

            \item Update $\lambda = (\lambda_i, \ldots ,\lambda_K)$ using a moving average:
            \begin{equation*}
               \lambda_{\text{new}} = \alpha \lambda_{\text{old}} + (1-\alpha) \hat{\lambda}_{\text{new}}. 
            \end{equation*}
        \end{enumerate}
        
        \item Update network weights $\mathrm{W}$ via gradient descent with learning rate $\gamma$.
    \end{enumerate}
\end{enumerate}
\end{algorithm}

In general, we are interested in determining the optimal values of the neural weights, $\mathrm{W}^{*}$, such that $f_{a}((x,\theta);\mathrm{W}^{*}) \approx u(x,\theta)$. This is done by minimizing \eqref{eq:loss_pde} using a stochastic gradient-based method \cite{kingma2015adam}. In this work, we combine the algorithm proposed by \cite{dgm} with the approach of \cite{guide_pinn}. The latter proposes a training pipeline for PINNs that addresses issues such as unbalanced backpropagation and causality violation.

 \cref{algo1} presents the methodology for training the neural network used in this work when the loss function consists of $K$ different components, which may include boundary conditions, initial conditions, source terms, or multiple differential operators. In this work, we fixed the dataset from the beginning, as we observed that this improves the accuracy of the neural network in our experiments. However, for very high-dimensional problems, this approach may become problematic, and it can be more effective to resample the data at each iteration of the algorithm as in the original pipeline of the DGM. The loss weights are adaptively updated to ensure that the different components of the loss function evolve at similar rates, preventing the optimization process from being biased toward any specific term.

\subsubsection{Neural Network Architecture}

The architecture, $a$, used in this work to define the mapping $f_{a}^\mathrm{ENC}(\cdot;\mathrm{W})$, 
is a modification of the feed-forward neural network inspired by the neural attention mechanism and it
was proposed by \cite{guide_pinn}. It consists of two encoders that project the inputs into higher dimensions, and these encoders then modify the hidden layers through pointwise multiplication in a residual connection style. Let us consider two encoders $V, U: \bar{D} \times \mathcal{T} \rightarrow \mathbb{R}^{d_{n}}$ and define the following operation:
\begin{equation}\label{eq:modiact}
    \tilde{\sigma}_{\sigma}(y) = \sigma(y) \odot U(z) + (1 - \sigma(y)) \odot V(z),
    \quad V(z) = \sigma  (C_{V}(z)), \quad U(z) = \sigma (C_{U}(z)), 
\end{equation}
where $z = (x, \theta)$, $\odot$ denotes element-wise multiplication and the subscript in $ \tilde{\sigma}_{\sigma}
$ indicates the activation function used; for example, $\tilde{\sigma}_{\tanh}$ means that $\sigma(\cdot) = \tanh(\cdot)$ is applied in the operation. Then, the mapping $f_{a}^\mathrm{ENC}(\cdot;\mathrm{W})$ is defined by an architecture specified through $a = (N_{n}, \tilde{\sigma}_{\sigma})$, following a composition rule analogous to that presented in \eqref{eq:nn_composition}, where the network dimensions are chosen as $N_{n} = (d_{x}+d_{\theta}, d_{n}, \ldots, d_{n}, d_{L}) \in \mathbb{N}^{L+1}$ and the activation function is fixed as $\sigma(\cdot) = \tanh(\cdot)$. The trainable parameters are collected in $\mathrm{W} = \{(\mathrm{W}_{V}, b_{V}), (\mathrm{W}_{U}, b_{U}), (\mathrm{W}_{l}, b_{l})_{l=1}^{L}\}$.

It has been observed that deep neural networks tend to learn low-frequency components of a target function first, a phenomenon known as spectral bias. This was noted in the work of \cite{spectral_bias}, where the authors analyzed ReLU neural networks using Fourier analysis. The problem was later revisited through the framework of the Neural Tangent Kernel (NTK), a tool for analyzing the training dynamics of neural networks \cite{ntk}. Spectral bias is also present in vanilla PINNs; therefore, to mitigate this issue, \cite{fourier_embeddings,eigenbias, guide_pinn} proposed adding random Fourier feature embeddings before passing inputs through the architecture described above. This embedding $\gamma: \mathbb{R}^{d_{x}+d_{\theta}} \rightarrow \mathbb{R}^{d_{F}+d_{\theta}}$ is defined as:
\begin{equation*}
    \gamma(z) = [\cos(Bx),\sin(Bx),\theta],
\end{equation*}
where $B \in \mathbb{R}^{\frac{d_{F}}{2} \times d_{x}}$ is sampled from a Gaussian distribution $\mathcal{N}(0, \sigma_{FF}^{2})$, and it is recommended that $\sigma_{FF}^{2} \in [1, 10]$. Note that this embedding is applied to the spatio-temporal collocations only. Therefore, when using these embeddings, the neural network with Fourier Embeddings is $f_{a}^\mathrm{FF}(z;\mathrm{W}) 
= f_{a}^\mathrm{ENC}(\gamma(z);\mathrm{W})$.

Although the training process and neural network architecture described above is effective, neural networks typically need large amounts of data to perform well and are prone to overconfidence in their predictions, even when those predictions are far from accurate. This issue becomes especially critical when neural networks are employed as surrogate models, particularly in cases involving out-of-distribution inputs. Therefore, obtaining well-calibrated uncertainty estimates is essential, as they offer a key way to recognize when the predictions of neural networks may not be reliable. In the following section, we present a methodology that enables neural networks to incorporate uncertainty through a fast approximation, allowing us to capture meaningful uncertainty in out of domain interpolation of the PDE solution.

\subsubsection{DeepGaLA}

Uncertainty quantification (UQ) has been extensively studied for neural networks used as function approximators \cite{nn_bayesian,nn_bayesian2}. More recently, research has also explored UQ for neural PDE solvers, such as Physics-Informed Neural Networks (PINNs) and DeepONets \cite{UQPINNS}. In this section, we introduce DeepGaLA, a random neural surrogate that combines the DGM methodology with a Laplace Approximation to efficiently approximate the forward model, $\mathcal{G}_{\mathrm{X}}$, while providing uncertainty estimates.

To equip the neural network with uncertainty estimation, we adopt a Bayesian approach for the weights $\mathrm{W}$ in $f_{a}(\cdot;\mathrm{W})$. Our objective is then to compute the posterior distribution $p(\mathrm{W} | \mathcal{D})$. This posterior is conditioned on $\mathcal{A}$ and $\mathcal{B}$ from
\eqref{eq:pde}, as these terms define the likelihood. The dataset $\mathcal D$ consists of collocation points sampled from the distributions $\pi_{D}, \pi_{\partial D}$, and $\pi_{\mathcal{T}}$ as in section \ref{subsubsec:neuralPDE}. We interpret the loss in
\eqref{eq:loss_pde} as a negative log-likelihood by modeling Gaussian noise around the fictitious data observed at the collocation points. The likelihood function is given by:
\begin{equation*}
p(\mathcal{D}|\mathrm{W})= p( \mathcal{D}_{\mathcal A}|\mathrm{W})p(\mathcal{D}_{\mathcal B}|\mathrm{W}) \propto  \prod_{i=1}^{N_{\mathcal A}} \exp \left(- \frac{\mathcal{L}_{\mathcal A}(\mathcal{D}_{\mathcal A,i};\mathrm{W})}{2 \sigma_{\mathcal A}^{2}}\right) \prod_{i=1}^{N_{\mathcal B}} \exp\left(- \frac{\mathcal{L}_{\mathcal B}(\mathcal{D}_{\mathcal B,i};\mathrm{W})}{2 \sigma_{\mathcal B}^{2}}\right),
\end{equation*}
where $\sigma_{i}^{2} = \sigma_{dG}^{2}/ \lambda_{i}$ for $i \in  {\mathcal A,\mathcal B}$. We assume that the likelihoods of the location $ \mathcal{D}_{\mathcal A}$ and $\mathcal{D}_{\mathcal B}$ are conditionally independent given $\mathrm{W}$. To complete the Bayesian formulation, we introduce a Gaussian prior $p(\mathrm{W})$. Applying Bayes' theorem, we obtain the posterior distribution $p(\mathrm{W}|\mathcal{D}) \propto p(\mathcal{D}|\mathrm{W}) p(\mathrm{W})$ which in general is intractable.

The Laplace Approximation (LA) \cite{laplace} is a well-known methodology for approximating the generally intractable posterior distribution with a Gaussian centered at the maximum a posteriori (MAP) solution, $\mathrm{W}_{\text{MAP}}$, and has been successfully applied to deep learning methodologies \cite{mackay_1992,laplace_redux}. DeepGaLA employs a Laplace approximation to provide uncertainty quantification for neural network predictions and proceeds in two stages. First, a neural network is trained to minimize the loss function \eqref{eq:loss_pde} using  \cref{algo1}, thereby obtaining the MAP estimate, $\mathrm{W}_{\text{MAP}}$. Second, a local Gaussian distribution is fitted around $\mathrm{W}_{\text{MAP}}$ , where the covariance matrix is given in terms of the Hessian of the negative log-posterior evaluated at $\mathrm{W}_{\text{MAP}}$.
 By considering a weight decay regularizer, which corresponds to a Gaussian prior distribution $p(\mathrm{W}) =\mathcal{N} (\mathrm{W} ;0,\gamma^{2} I)$ \cite{laplace_redux}, the inverse covariance matrix takes the form
\begin{equation*}
\label{eq:lambda}
  -\nabla_{\mathrm{W}}^{2}\text{log }p(\mathrm{W}|\mathcal{D}) = - 
 \frac{1}{2\sigma_{\mathcal A}^2} \sum_{i=1}^{N_{\mathcal A}} \nabla_{\mathrm{W}}^{2} \text{log } p(\mathcal{D}_{\mathcal A,i}|{W})|_{\mathrm{W}_{\text{MAP}}}-\frac{1}{2\sigma_{\mathcal B}^2}\sum_{i=1}^{N_{\mathcal B}}
   \nabla_{\mathrm{W}}^{2} \text{log } p(\mathcal{D}_{\mathcal B,i}|{\mathrm{W}})|_{\mathrm{W}_{\text{MAP}}}+\gamma^{-2}I.
\end{equation*}
However, computing this matrix can be computationally expensive, as its complexity scales quadratically with the number of parameters. It has been shown that randomizing the weights and computing a LA on just the last layer, $\mathcal{W}_{L}=(\mathrm{W}_{L},b_{L})$, of neural networks can achieve good results and it is often sufficient to produce reliable uncertainty estimates, as noted by \cite{pmlr-v206-sharma23a}. Hence,we treat only the last-layer weights $\mathcal{W}_{L}$ probabilistically to obtain the posterior distribution $p(\mathcal{W}_{L}|\mathcal{D})$.
Then, we compute the Hessian matrix using the generalized Gauss-Newton (GGN) approximation, a widely used technique for efficiently approximating the Hessian in deep learning settings \cite{pmlr-v206-sharma23a}. The precision matrix $\Lambda^{-1} \in \mathbb{R}^{d_{L} \times d_{L}}$ of $\mathcal{W}_{L}$ is then given by:
\begin{equation}
\label{hessian_eq}
\Lambda^{-1} =    \frac{1}{\sigma_{\mathcal A}^{2}}\sum_{i=1}^{N_{\mathcal A}}J_{\mathcal{W}_{L}}\left(\mathcal A_{i}\right)
J_{\mathcal{W}_{L}}\left(\mathcal A_{i}\right)^{T}|_{\mathrm{W}_{\text{MAP}}} + \frac{1}{\sigma_{\mathcal B}^{2}} \sum_{i=1}^{N_{\mathcal B}}J_{\mathcal{W}_{L}}\left(\mathcal B_{i}\right)
J_{\mathcal{W}_{L}}\left(\mathcal B_{i}\right)^{T}|_{\mathrm{W}_{\text{MAP}}} + \gamma^{-2}I,
\end{equation}
where $J_{\mathcal{W}_{L}}(\cdot) \in \mathbb{R}^{d_{L}}$ represents the Jacobian with respect to the weights of the final layer, $\mathcal{W}_{L}$. The entries of these Jacobians are given by $J_{\mathcal{W}_{L},i} \left(\mathcal A_{l}\right) = \frac{\partial }{\partial \mathcal{W}_{L}^{i}}\mathcal{A}_{\theta_{\mathcal A,l}}\left[f_{a}(\mathcal{D}_{\mathcal A,l}; \mathrm{W})\right]$ and $
J_{\mathcal{W}_{L},i} \left(\mathcal B_{l}\right) = \frac{\partial }{\partial \mathcal{W}_{L}^{i}}\mathcal{B}_{\theta_{\mathcal B,l}}\left[f_{a}( \mathcal{D}_{\mathcal B,l}; \mathrm{W})\right]$. 

The final step required to characterize the posterior distribution over the last-layer weights consists of optimizing the hyperparameters $\sigma^{2}_{dG}$ and $\gamma$, which in this work are initially chosen by maximizing the log-likelihood. However, we observed that jointly optimizing both parameters causes the uncertainty estimates produced by DeepGaLA to shrink very rapidly. To address this, we optimize $\gamma$ while fixing $\sigma^{2}_{dG} = 1$, which leads to more reliable uncertainty estimates. For further details, we refer the reader to Appendix~\ref{appendix:parameters}. Once $\gamma$ has been tuned, the posterior predictive distribution can be computed by exploiting the fact that the neural network is linear in the last layer, allowing a Gaussian approximation of the output. Then, for any test input $z_i=(x_i,\theta_i)$, the predictive distribution is given by:
\begin{equation*}
p(f_{a}(z)|\mathcal{D} ) = \mathcal{N}\left(f_{a}(z_i;\mathrm{W}_{\text{MAP}}),\Phi^{(L-1)}(z_i;\mathrm{W}_{\text{MAP}})^{T} \Lambda \Phi^{(L-1)}(z_i;\mathrm{W}_{\text{MAP}})\right),
\end{equation*}
where $f_{a}(z_i;\mathrm{W}_{\text{MAP}})$ denotes the network output at the MAP estimate, and $\Phi^{(L-1)}(z_i;\mathrm{W}_{\text{MAP}})$  is defined according to the composition rule in \eqref{eq:nn_composition}. In the case of multiple test inputs $z = \{z_{i}\}_{i=1}^{d_{y}}$, the model output is given by a multivariate Gaussian distribution:
\begin{equation}
 p(f_{a}(z)|\mathcal{D} ) = \mathcal{N}\left(m^{f}_\mathrm{dG,N}(z),\Gamma_\mathrm{dG,N}(z,z')\right),
\end{equation}
where
\begin{equation}
\begin{aligned}
m^{f}_\mathrm{dG,N}(z) &= \left(f_{a}(z_{1};\mathrm{W}_{\text{MAP}}),\dots,f_{a}(z_{d_{y}};\mathrm{W}_{\text{MAP}})\right), \\
\Gamma_\mathrm{dG,N}(z,z') &= 
\begin{bmatrix}
\Phi^{(L-1)}(z_{1};\mathrm{W}_{\text{MAP}})^{T}\Lambda \Phi^{(L-1)}(z_{1};\mathrm{W}_{\text{MAP}}) & \cdots & \Phi^{(L-1)}(z_{1};\mathrm{W}_{\text{MAP}})^{T}\Lambda \Phi^{(L-1)}(z_{d_{y}};\mathrm{W}_{\text{MAP}}) \\
\vdots & \ddots & \vdots \\
\Phi^{(L-1)}(z_{d_{y}};\mathrm{W}_{\text{MAP}})^{T}\Lambda \Phi^{(L-1)}(z_{1};\mathrm{W}_{\text{MAP}}) & \cdots & \Phi^{(L-1)}(z_{d_{y}};\mathrm{W}_{\text{MAP}})^{T}\Lambda \Phi^{(L-1)}(z_{d_{y}};\mathrm{W}_{\text{MAP}})
\end{bmatrix}.
\end{aligned}
\end{equation}
This formulation leverages the structure of deep neural networks to enable tractable and scalable uncertainty estimation. Although DeepGaLA has been presented using feedforward neural networks, $f_{a}(\cdot;\mathrm{W})$ , the framework can be can be applied to architectures such as $f_{a}^{\mathrm{ENC}}(\cdot;\mathrm{W})$ and $f_{a}^{\mathrm{FF}}(\cdot;\mathrm{W})$.
With this formulation, we can construct a random neural network surrogate for $\mathcal{G}_{\mathrm{X}}$, which we denote by $\mathcal{G}_{\mathrm{X}}^{\mathrm{dG,N}}$ and follows a multivariate Gaussian distribution with mean $m^{\mathcal{G}}_\mathrm{dG,N}$ and covariance $\Gamma_\mathrm{dG,N}$.

\subsection{Gaussian Process Regression}

An alternative approach for constructing a random surrogate model is to use a nonparametric method, such as Gaussian Process (GP) regression. GP regression inherently provides uncertainty estimates when approximating a function, making them an ideal benchmark for assessing the performance of DeepGaLA. 

For simplicity, we first consider the approximation of scalar-valued functions $g: \mathcal{T} \rightarrow \mathbb{R}$. The first step is to place a Gaussian prior on $g$ in the following way:
\begin{equation}
\label{eq:gp_real}
    g_0 \sim \text{GP}(m(\theta), k(\theta, \theta')),
\end{equation}
where $m: \mathcal{T} \rightarrow \mathbb{R}$ is the mean function and $k: \mathcal{T} \times \mathcal{T} \rightarrow \mathbb{R}$ is the kernel function. The choice of kernel reflects assumptions about the smoothness and structure of the function $g$; among the most commonly used kernels are the Matérn and squared exponential Gaussian kernels \cite{matern_spatial,Rasmussen2004}.
Assume that we have access to $N$ design points 
$\Theta = \{ \theta_i \}_{i=1}^{N} $ with corresponding function values $g(\Theta) = [g(\theta_1), \ldots, g(\theta_{\mathrm{N}})] \in \mathbb{R}^{N}$. Denote by $ g_{\mathrm{N}}$ the Gaussian process conditioned on the observed values $g(\Theta)$, also called the posterior predictive. Then,
\begin{equation}
    g_{\mathrm{N}} \sim \text{GP}(m_{\mathrm{N}}^{g}(\theta), k_{\mathrm{N}}(\theta, \theta')),
\end{equation}
where the predictive mean and covariance functions are given by
\begin{equation}
\label{eq:gp_pred}
\begin{split}
    &m_{\mathrm{N}}^{g}(\theta) = m(\theta) + k(\theta, \Theta) K(\Theta, \Theta)^{-1} (g(\Theta) - m(\Theta)), \\
    &k_{\mathrm{N}}(\theta, \theta') = k(\theta, \theta') - k(\theta, \Theta)^{\top} K(\Theta, \Theta)^{-1} k(\theta', \Theta).
\end{split}
\end{equation}
Here, $k(\theta, \Theta) = [k(\theta, \theta_1), \ldots, k(\theta, \theta_{\mathrm{N}})] \in \mathbb{R}^{N}$, and $K(\Theta, \Theta) \in \mathbb{R}^{{N} \times {N}} $ is the kernel matrix evaluated at the design points, with entries $K(\Theta, \Theta)_{ij} = k(\theta_i, \theta_j)$. 

One of the interests when employing GPs for regression tasks is to assess how well the GP can approximate a function as more design points are used. A key concept that measures how well a set of design points fills the space $\mathcal{T}$ is the fill distance $h_{\Theta}$, defined as
\begin{equation}
    h_{\Theta} = \sup_{\theta \in \mathcal{T}} \inf_{\theta_{n} \in \Theta} \| \theta - \theta_{n} \|_{2}.
\end{equation}
The fill distance can be interpreted as the largest gap between any point in $\mathcal{T}$ and its closest design point in $\Theta$. We are now ready to state a convergence theorem from \cite{Wendland_2004}, which guarantees convergence of $m_{\mathrm{N}}^g$ to the true function $g$.

\begin{proposition}
\label{prop:gp_rate}
Suppose $\mathcal{T} \subseteq \mathbb{R}^{d_{\theta}}$ is a bounded Lipschitz domain that satisfies an interior cone condition, and the symmetric positive definite kernel $k$ is such that its Reproducing kernel Hilbert space $H_{k}$ is isomorphic to the Sobolev
space $H^{\tau}(\mathcal{T})$, with $\tau = n+r$, $n \in \mathbb{N}$, $n> d_{\theta}/2$ and $0 \leq r < 1$. Suppose $m_{\mathrm{N}}^{g}$ is given by \eqref{eq:gp_pred} with $m=0$. If $g \in H^{\tau}(\mathcal{T})$, then there exists a constant $C$, independent of $g$, $\Theta$ and $N$, such that
\begin{equation*}
    \|  g - m_{\mathrm{N}}^{g}\|_{H^{\beta}(\mathcal{T})} \leq Ch_{\Theta}^{\tau-\beta}\|g\|_{H^{\tau}(\mathcal{T})}, \quad \text{for any }\beta \leq \tau,
\end{equation*}
for all sets $\Theta$ with $h_{\Theta}$ sufficiently small. 
\end{proposition}

\cref{prop:gp_rate} indicates that, as $d_{\theta}$ increases, Gaussian processes become less effective at approximating the function $g$, since the fill distance $h_{\Theta}$ usually scales as $N^{-1/d_{\theta}}$ (see e.g. \cite{teckentrup} and the references therein).

\subsubsection{Gaussian Process Regression as a Surrogate}

So far, we have presented GP regression to approximate real-valued functions. 
However, to construct a surrogate for $\mathcal{G}_{\mathrm{X}}$, the framework presented so far must be extended to a multi-output GP. 
This extension can be achieved by considering a prior over $\mathcal{G}_{\mathrm{X}}$ of the form:
\begin{equation*}
    \mathcal{G}_{\mathrm{X},0} \sim \text{GP}(0, K(\theta, \theta')),
\end{equation*}
where we have chosen zero mean for simplicity and 
$K: \mathcal{T} \times \mathcal{T} \rightarrow \mathbb{R}^{d_{y} \times d_{y}}$ is the covariance kernel. 
The kernel is often considered in the form $K(\theta,\theta') = k_p(\theta,\theta') I_{d_{y}}$, 
where $I_{d_{y}} \in \mathbb{R}^{d_{y} \times d_{y}}$ is the identity matrix and $k_p(\theta,\theta'): \mathcal{T} \times \mathcal{T} \rightarrow \mathbb{R}$ is a scalar-valued kernel in parameter space. 
This assumption reduces the multi-output GP to independent GPs for each output dimension, as in \eqref{eq:gp_real}. 
The predictive mean and covariance can be formulated analogously to \eqref{eq:gp_pred}.

An important remark when constructing a surrogate for $\mathcal{G}_{\mathrm{X}}$ is that this function depends not only on $\theta$ but also on the spatial observations $X$. 
To account for spatial correlation, the work of \cite{Bai_2024} provides a method to introduce spatial correlation and incorporate information from the PDE by extending the model of \cite{raissi2017machine} to parametric equations. We refer to this prior as a physics-informed Gaussian Process (PIGP), and briefly introduce this framework below. 

The starting point for the PIGP prior is a (scalar-valued) prior on the PDE solution $u$ of the form 
\begin{equation*}
   u_0(x,\theta) \sim \text{GP}(0, k_{p}(\theta,\theta') \, k_{s}(x,x')). 
\end{equation*}
If the differential operator in Equation~\eqref{eq:pde} is linear, 
and for the sets $X_{f} = \{x_{i}\}_{i=1}^{d_{f}} \subseteq D$ and 
$X_{g} = \{x_{i}\}_{i=1}^{d_{g}} \subseteq \partial D$, 
the (joint) PIGP prior is
\begin{equation}
\label{eq:pigp_prior}
U(\theta)=
\begin{bmatrix}
u(X,\theta) \\
g( X_g,\theta) \\
f(X_f,\theta)
\end{bmatrix}
\sim \text{GP}\!\left(
0,\;
K(\theta,\theta')
\right),
\end{equation}
where $K(\theta,\theta') \in \mathbb{R}^{d_y+d_f+d_g \times d_y+d_f+d_g} $ is given by
\begin{equation*}
K(\theta,\theta') = k_p(\theta,\theta') 
\begin{bmatrix}
K_s(X, X) & \mathcal{B}_{\theta_{\mathcal{B}}'} K_s(X, X_g) & \mathcal{A}_{\theta_{\mathcal{A}}'} K_s(X, X_f) \\
\mathcal{B}_{\theta_{\mathcal{B}}} K_s(X_g, X) & \mathcal{B}_{\theta_{\mathcal{B}}} \mathcal{B}_{\theta_{\mathcal{B}}'} K_s(X_g, X_g) & \mathcal{B}_{\theta_{\mathcal{B}}} \mathcal{A}_{\theta_{\mathcal{A}}'} K_s(X_g, X_f) \\
\mathcal{A}_{\theta_{\mathcal{A}}} K_s(X_f, X) & \mathcal{A}_{\theta_{\mathcal{A}}} \mathcal{B}_{\theta_{\mathcal{B}}'} K_s(X_f, X_g) & \mathcal{A}_{\theta_{\mathcal{A}}} \mathcal{A}_{\theta_{\mathcal{A}}'} K_s(X_f, X_f)
\end{bmatrix}.
\end{equation*}
Then
\begin{equation}
\label{eq:pigp_marginal}
\mathcal{G}_{\mathrm{X},N} = \mathcal{G}_{\mathrm{X}}(\theta) \,\big|\, u(X,\Theta), g(X_g,\Theta), f( X_f,\Theta) 
\sim \text{GP}\left(
m_{N, X_f, X_g}^{u}(\theta),
K_{N, X_f, X_g}(\theta, \theta')
\right),
\end{equation}
where the marginal predictive mean and covariance are given by
\begin{equation}
\label{eq:pigp_pred}
\begin{split}
&m_{N, X_f, X_g}^{u}(\theta) 
= 
\begin{bmatrix}
K_{uu}(\theta, \Theta) & K_{ug}(\theta, \Theta) & K_{uf}(\theta, \Theta)
\end{bmatrix}
K(\Theta, \Theta)^{-1}
 U(\Theta),  \\
&K_{N, X_f, X_g}(\theta, \theta') 
= K(\theta, \theta') 
- 
\begin{bmatrix}
K_{uu}(\theta, \Theta) & K_{ug}(\theta, \Theta) & K_{uf}(\theta, \Theta)
\end{bmatrix}
K(\Theta, \Theta)^{-1}
\begin{bmatrix}
K_{uu}(\theta', \Theta) \\[2pt]
K_{ug}(\theta', \Theta) \\[2pt]
K_{uf}(\theta', \Theta)
\end{bmatrix},
\end{split}
\end{equation}
where $K_{uu}(\theta, \Theta) = k_p(\theta,\Theta)\,K_s(X, X) \in \mathbb{R}^{d_{y} \times d_{y} N},$ $K_{ug}(\theta, \Theta) = k_p(\theta,\Theta)\,\mathcal{B}_{\theta_{\mathcal{B}}} K_s(X, X_g) \in \mathbb{R}^{d_{y} \times d_{g} N},$ $K_{uf}(\theta, \Theta) = k_p(\theta,\Theta)\,\mathcal{A}_{\theta_{\mathcal{A}}} K_s(X, X_f) \in \mathbb{R}^{d_{y} \times d_{f} N}, $ and $K(\Theta, \Theta) \in \mathbb{R}^{N(d_{y}+d_{f}+d_{g}) \times N(d_{y}+d_{f}+d_{g})}.$

For more details on this model, we refer the reader to \cite{Bai_2024}. With the exposition so far, we can now construct a surrogate for $\mathcal{G}_{\mathrm{X}}$, which we denote by $\mathcal{G}_{\mathrm{X}}^\mathrm{{GP,N}}$ when using a GP surrogate with mean $m_\mathrm{{GP,N}}^{\mathcal{G}}$ and covariance $\Gamma_{_\mathrm{{GP,N}}}:= K_\mathrm{{N}}$. In the case where we use a PIGP, it should be understood that $\mathcal{G}_{\mathrm{X}}^\mathrm{{PIGP,N}}$  represents the surrogate with mean and covariance given by Equations \eqref{eq:pigp_pred} and denoted by $m_\mathrm{{PIGP,N}}^{\mathcal{G}}$ and $\Gamma_{_\mathrm{{PIGP,N}}}:=K_\mathrm{{PIGP,N}}$, respectively.

\section{Sampling from the posterior }
\label{sec:4}

This section is devoted to reviewing algorithms for generating samples from the mean and marginal approximated posterior distributions obtained from a random surrogate model. In particular, we review the Metropolis–Hastings (MH) algorithm for generating samples from a target density, $\pi(\theta)$ \cite{metropolis,hastings,robert1999monte}. The second part presents an approach for assessing the accuracy of posterior approximations in situations where computing the true posterior is prohibitively expensive. The section concludes with an analysis of this methodology and provides convergence rates for the case of Gaussian process regression.

\subsection{Metropolis-Hastings Algorithm}

Given an initial state $\theta_0$, such that $\pi(\theta_{0})>0$, and a proposal density $q(\tilde{\theta} |\theta_j)$, the MH algorithm generates samples $\{\theta_j\}_{j=1}^\infty$, such that  $\theta_j \sim \pi(\theta)$ as $j\rightarrow \infty$. \cref{algo:mcmc} describes the MH algorithm used to generate a finite set of samples $\{\theta_{j}\}_{j=1}^{N_{S}}$ from the target distribution. In practice, the first $n_{0}$ samples are discarded, as they may still be far from the target distribution. These samples constitute the so called burn-in of the algorithm. A common rule of thumb is to set the burn-in period to approximately $10\%$ of the total samples. Therefore, the samples used for inference are $\{\theta_{j}\}_{j=n_{0}+1}^{N_{S}}$.

\begin{algorithm}
\caption{Metropolis-Hastings MCMC}\label{algo:mcmc}
\textbf{Input: } Initial state $\theta_{0}$, proposal distribution $q(\cdot|\theta_{j})$, target distribution $\pi(\cdot)$ and number of iterations $N_{S}$.

\textbf{Output: } $\{\theta_{j}\}_{j=1}^{N_{S}}$.

\begin{enumerate}
\renewcommand{\labelenumi}{}
 \item \textbf{for} $j = 0$ \textbf{to} $N_{S}-1$ \textbf{do}
    \begin{enumerate}
        \item Given $\theta_{j}$,  generate new candidate $\tilde{\theta}$ from $q(\Tilde{\theta}| \theta_{j})$.
        \item Accept $\tilde{\theta}$ as a sample with probability :
        \begin{equation}\alpha_{\mathrm{MH}}(\Tilde{\theta}|\theta_{j}) = \min \left(1, \frac{\pi(\Tilde{\theta})q(\theta_{j}|\Tilde{\theta}) }{\pi(\theta_{j})q(\Tilde{\theta}|\theta_{j})} \right) , 
        \end{equation}
        i.e $\theta_{j+1}= \tilde{\theta}$  with probability $\alpha(\Tilde{\theta}|\theta_{j})$ and  $\theta_{j+1}=\theta_{j}$ with probability $1-\alpha(\Tilde{\theta}|\theta_{j})$.
    \end{enumerate}
\end{enumerate}     
\end{algorithm}

Despite its apparent simplicity, selecting an appropriate proposal distribution is crucial and often challenging. The random-walk proposal is a common choice.  
It uses a Gaussian proposal distribution centered at \(\theta_j\) with covariance \(\beta^{2} I_{d_\theta}\), i.e. $q(\tilde{\theta}\mid \theta_j)=\mathcal{N}\big(\theta_j;\beta^{2} I_{d_\theta}\big)$. Equivalently, the proposed move can be written as follows,
\begin{equation*}\label{rwmh}
\tilde{\theta} = \theta_j + \beta \,\xi,\qquad \xi \sim \mathcal{N}(0,I_{d_\theta}),\qquad \beta>0.
\end{equation*}
The random walk proposal is symmetric, which means that $q(\theta_{j} | \tilde{\theta}) = q(\tilde{\theta} | \theta_{j})$ and the acceptance probability depends only on the target distribution,
\begin{equation}
\alpha_{\mathrm{MH}}(\tilde{\theta} \mid \theta_j) =\min\!\left( 1, \dfrac{\pi(\tilde{\theta})}{\pi(\theta_j)} \right).
\label{eq:acceptance_rw}
\end{equation}
The parameter $\beta$ is often referred to as the proposal variance, and its choice should not be overlooked. If $\beta$ is too small, the parameter space is explored inefficiently due to very small moves. On the other hand, if $\beta$ is too large, the algorithm rejects proposals more frequently, which also leads to poor exploration. In both cases, the resulting Markov chain can exhibit large variance and slow mixing. The optimal choice of $\beta$ has been studied in \cite{optimal_beta}, where it was shown that  for Gaussian target distributions that $\beta$ should be tuned to achieve an acceptance rate of approximately $0.234$, which minimizes the asymptotic variance. Although this result was derived under specific assumptions, it is widely used as a practical rule of thumb. We refer to the value of $\beta$ that yields an acceptance rate close to $0.234$ as the optimal proposal variance.

It is worth noting that the MH algorithm allows for a wide variety of proposal distributions. In particular, when the posterior distribution is differentiable with respect to the parameters, proposals can be designed that take advantage of the gradient information to improve efficiency. A well-known example is based on the Langevin equation resulting in the Metropolis-adjusted Langevin algorithm (MALA) \cite{mala}. 

As seen in Equation \eqref{eq:acceptance_rw}, even with the simplifications provided by the random-walk proposal, the data likelihood must generally still be evaluated at each iteration until convergence. This process can become computationally expensive in inverse problems in PDEs due to the need to compute the forward model $\mathcal{G}_{\mathrm{X}}$, which often requires a numerical approximation such as the Finite Element Method or Spectral Methods. Therefore, when employing surrogates to alleviate this burden, the objective is to draw samples from the mean and marginal approximate posteriors, given by Equations \eqref{eq:mean-approx} and  \eqref{eq:marginal-approx}. However, the question remains whether these approximate posteriors are close to or far from the true posterior. In the next section, we discuss this in more detail.

\subsection{Evaluation of Posterior Approximations}

Controlling the error between the approximate posterior and the true posterior is a desirable property when using surrogate models. In the case of numerical methods, such as the Finite Element Method (FEM), it has been shown that the approximate posterior converges to the true posterior as the discretization of the observational operator is refined, provided that the MCMC algorithm has converged \cite{approx_inversep,hoang2013complexity}. For random surrogate models, one ideally seeks that both the mean and marginal approximate posterior, $\pi^{y,\mathcal{G}_{\mathrm{X}}^{\mathrm{S,N}}}_{\mathrm{mean}}(\theta)$ and  
$\pi^{y,\mathcal{G}_{\mathrm{X}}^{\mathrm{S,N}}}_{\mathrm{marginal}}$, 
converge to the true posterior $\pi^{y}$ as $N \to \infty$. For GP surrogates based on standard kernels as defined in equations \eqref{eq:gp_pred}, it has indeed been shown that the approximate posteriors, induced by the mean and marginal approximations, converge to the true posterior in the Hellinger distance \cite{aretha_posterior,kostas_aretha}. However, in practice, one often has access to only a limited number of design points. In such cases, the approximate posterior induced by the surrogate may still be far from the true posterior. On the other hand, surrogate models such as neural networks are less theoretically understood, and there is no clear guarantee of convergence to the true posterior. 

In this section, we describe a computable quantity $\alpha_{\mathrm{val}}$ that can be used to assess the accuracy of the approximate posterior distributions. The quantity is derived from the Delayed Acceptance MCMC (DA-MCMC) algorithm \cite{delayed_mcmc,multilevel_da}, which employs a two-stage procedure to generate samples from a target distribution: the first stage uses a coarse, computationally inexpensive model to pre-screen proposed candidates, which are then evaluated with a finer, more accurate but costly model, thereby avoiding unnecessary evaluations of the fine model. The main benefit is that the quantity $\alpha_{\mathrm{val}}$ can be computed with a much smaller number of evaluations of the fine model compared to a full exploration of the posterior distribution with the fine model.

\cref{algo2} presents the approach to validate the approximate posterior distribution using the DA-MCMC algorithm. The central idea is to assess whether samples accepted for the approximate posterior by the MH algorithm are also consistent with the true posterior. The true posterior is approximated by a reference solution, denoted by $\pi^{y,F}$, obtained from a sufficiently accurate numerical approximation of the forward model, $\mathcal{G}_{\mathrm{X}}^\mathrm{F}$. Assume that we want to evaluate the accuracy of the surrogate-approximated posterior with a budget of $N_F$ evaluations of the fine model. Starting from the final sample $\theta_{N_{S}}$ of the surrogate posterior produced by Algorithm \ref{algo:mcmc}, we run the DA-MCMC algorithm until the fine model has been evaluated $N_F$ times. With $N_\mathrm{acc}$ the number of proposals tested and accepted by the fine model, we define $\alpha_{\mathrm{val}}$ as:
\begin{equation}
\label{eq:ar}
\alpha_{\mathrm{val}} = \frac{N_\mathrm{acc}}{N_F}.
\end{equation}
\begin{algorithm}
\caption{Posterior Validation via Delayed Acceptance MCMC}\label{algo2}
\textbf{Input: } Initial state $\theta=\theta_{N_{S}}$ given by Algorithm \ref{algo:mcmc}, surrogate posterior approximation $\pi^{y,S}(\cdot)$, reference posterior $\pi^{y,F}(\cdot)$, proposal distribution $q(\cdot|\theta)$ and number of evaluations of the fine model $N_{F}$.

\textbf{Output: } Acceptance rate $\alpha_{\mathrm{val}} = N_\mathrm{acc}/N_{F}$

\begin{enumerate}
\item Initialize $n_F = 0$.
\item \textbf{while} $n_F  < N_F$ \textbf{do}
\begin{enumerate}
    \item Given the current state $\theta$, generate a candidate $\tilde{\theta}$ from $q(\tilde{\theta} | \theta)$ and compute the \textbf{coarse acceptance probability}:
        \begin{equation*}
           \tilde{\alpha}(\tilde{\theta}|\theta) = \min \left(1, \frac{\pi^{y,S}(\tilde{\theta}) q(\theta|\tilde{\theta}) }{\pi^{y,S}(\theta) q(\tilde{\theta}|\theta)} \right). 
        \end{equation*}
    \item If $\tilde{\theta}$ is not accepted, return to step (a). Otherwise, update $n_F \leftarrow n_F + 1$ and compute the \textbf{fine acceptance probability}:
    \begin{equation}\label{eq:fine_prob}
        \alpha(\tilde{\theta}|\theta) = \min\left(1, \frac{\pi^{y,F}(\tilde{\theta})\, \pi^{y,S}(\theta)}{\pi^{y,F}(\theta)\, \pi^{y,S}(\tilde{\theta})} \right).
    \end{equation} 
    \item If $\tilde{\theta}$ is accepted in step (b), set $\theta = \tilde{\theta}$ and $ N_\mathrm{acc} \leftarrow N_\mathrm{acc} + 1.$
\end{enumerate}
\end{enumerate}     
\end{algorithm}
Intuitively, $\alpha_{\mathrm{val}}$ takes small values when the approximate posterior distribution $\pi^{y,S}$ associated with the surrogate $S$ is far from the reference posterior $\pi^{y,F}$, and conversely large values when the surrogate posterior distribution  is close the true posterior. As we will see in Section \ref{sec:5}, the value of $\alpha_{\mathrm{val}}$ provides a meaningful diagnostic also for small values of $N_F$.

In order to justify the use of $\alpha_{\mathrm{val}}$ as a convergence metric, we now analyze the fine acceptance probability $\alpha$ of \cref{algo2} and show that this converges to 1 as the error between our surrogate model and the reference solution $\mathcal{G}_{\mathrm{X}}^\mathrm{F}$ goes to zero. Note that if $\alpha \equiv 1$ in \cref{algo2} then $\alpha_{\mathrm{val}}=1$, and furthermore $\mathbb{E}[1-\alpha_{\mathrm{val}}] = \mathbb{E}\left[1- \frac{N_\mathrm{acc}}{N_F}\right] = \mathbb{E}[1-\alpha]$ since $\mathbb{E}[N_\mathrm{acc}] = \mathbb{E}[\alpha]N_F$. We are hence interested in analyzing the expected probability of rejection,
 $\mathbb{E}_{\zeta} \big[\, 1 - \alpha(\tilde{\theta}|\theta) \,\big],$
where $\zeta$ denotes the joint distribution of a current state $\theta$ and a proposal $\tilde{\theta}$. Note that the proposal $\tilde{\theta}$ has been accepted as a sample of $\pi^{y,S}$ in step (a), whereas, in the limit as the number of iterations $j \to \infty$, the current state $\theta$ is distributed according to $\pi^{y,F}$. Therefore, for simplicity, we assume that the marginal distributions of $\zeta$ are $\pi^{y,S}$ and $\pi^{y,F}$, respectively.

The following Theorem presents bounds for the expected probability of rejection when employing the mean and marginal approximations to the posterior, given in Equations \eqref{eq:mean-approx} and \eqref{eq:marginal-approx}, denoted by $\mathbb{E}_{\zeta}\left[ 1 - \alpha_{\mathrm{mean}}(\tilde{\theta} | \theta) \right]$ and $\mathbb{E}_{\zeta}\left[ 1 - \alpha_{\mathrm{marginal}}(\tilde{\theta} | \theta) \right]$, respectively.  The proof follows closely the work of \cite{aretha_mmcmc} and \cite{aretha_posterior}.\footnote{ The weighted $L^2$ norm of a function $g : \mathcal{T} \rightarrow \mathbb{R}^{d_y}$ is defined as $\|g\|_{L_{\pi^{y,F}}^{2}(\mathcal T)} = \int_{\mathcal{T}} \|g(\theta)\|_{2}^{2} \pi^{y,F}(\theta)\, d\theta = \sum_{i=1}^{d_{y}} \int_{\mathcal{T}} (g^{i}(\theta))^{2} \pi^{y,F}(\theta)\, d\theta.$}

\begin{theorem}
\label{prop:acceptance}

Denote by $\zeta$ the joint distribution of $\theta$ and $\tilde{\theta}$, with marginals $\pi^{y,F}$ and $\pi^{y,S}$, respectively, and denote by $\mathcal{G}_{\mathrm{X}}^{\mathrm{S,N}} \sim \mu^{\mathrm{S,N}}$ the random surrogate model, with mean $m_{\mathrm{S,N}}^{\mathcal{G}}$. Assume that $\sup_{\theta \in \mathcal{T}} \mathbb{E}_{\mu^{S,N}}\|\mathcal{G}_{\mathrm{X}}^{\mathrm{S,N}}(\theta)\|^{2}_{2} \leq C_{\mathrm{S}}^2$ and $\sup_{\theta \in \mathcal{T}} \|\mathcal{G}_{\mathrm{X}}^{\mathrm{F}}(\theta)\|_{2} \leq C_{\mathrm{F}}$, for some constants $C_\mathrm{S}$ and $C_\mathrm{F}$ independent of N.  Then there exist constants $C_{1}$ and $C_{2}$, independent of $N$, such that:

\begin{equation*}
    \begin{split}
&\mathbb{E}_{\zeta}\left[ 1-\alpha_{\mathrm{mean}}(\tilde{\theta} | \theta) \right] 
\leq C_{1}\| \mathcal{G}_{\mathrm{X}}^\mathrm{F} -  m_{\mathrm{S,N}}^{\mathcal{G}} \|_{L_{\pi^{y,F}}^{2}(\mathcal T)}, 
\quad \text{and} \quad \\ 
&\mathbb{E}_{\zeta}\left[ 1-\alpha_{\mathrm{marginal}}(\tilde{\theta} | \theta) \right] 
\leq C_{2} \left\| \left( \mathbb{E}_{\mu^{\mathrm{S,N}}}\left[ \| \mathcal{G}_{\mathrm{X}}^\mathrm{F} -  \mathcal{G}_{\mathrm{X}}^{\mathrm{S,N}}\|_{2}^{1+ \delta}\right] \right)^{1/(1+ \delta)} \right\|_{L_{\pi^{y,F}}^{2}(\mathcal T)}  \quad \text{for any} \quad 0 < \delta \leq 1.
\end{split}
\end{equation*}
\end{theorem}

\begin{pf}
    
It has been proved that the fine acceptance probability, \eqref{eq:fine_prob}, is given by \cite{preco_mcmc,multilevel_da} 
\begin{equation*}
    \alpha(\tilde{\theta} | \theta) 
=\min\left(1, \frac{\pi^{y,F}(\tilde{\theta})\, \pi^{y,S}(\theta)}{\pi^{y,F}(\theta)\, \pi^{y,S}(\tilde{\theta})} \right)= \min\left(1, \frac{l_{F}(\tilde{\theta}) \pi_{0}(\tilde{\theta}) l_{S}(\theta)\pi_{0}(\theta)}{l_{F}(\theta)\pi_{0} (\theta) l_{S}(\tilde{\theta})\pi_{0}(\tilde{\theta})} \right) =  \min\left(1, \frac{l_{F}(\tilde{\theta})\, l_{S}(\theta)}{l_{F}(\theta)\, l_{S}(\tilde{\theta})} \right)
\end{equation*}
where $l_{i}(\theta) = l_{i}(y \mid \theta)$, for $i \in \{S, F\}$, denotes the likelihood under the surrogate and fine models, respectively. 
Now, let us find a bound for the probability of rejection, $1-\alpha(\tilde{\theta}|\theta)$. If the ratio $\frac{l_{F}(\tilde{\theta})l_{S}(\theta)}{l_{F}(\theta)l_{S}(\tilde{\theta})} \geq 1$, then $1-\alpha(\tilde{\theta}|\theta)=0$. Otherwise, we have that
\begin{equation*}
\begin{split}
      1-\alpha(\tilde{\theta}|\theta) = \left( 1 - \frac{l_{F}(\tilde{\theta})}{l_{S}(\tilde{\theta})} \right) + \left(\frac{l_{F}(\tilde{\theta})l_{S}(\theta)}{l_{F}(\theta)l_{S}(\tilde{\theta})} \right)\left( 1-\frac{l_{F}(\theta)}{l_{S}(\theta)}\right)
      \leq \left| 1 - \frac{l_{F}(\tilde{\theta})}{l_{S}(\tilde{\theta})} \right| + \left| 1-\frac{l_{F}(\theta)}{l_{S}(\theta)}\right|.
\end{split}
\end{equation*}
The last expression is valid for both approximations to the posterior distribution, the mean and the marginal. 
Let us first focus on the mean approximation to the posterior distribution and consider either of the last two terms. 
By using the definition of the likelihood, we have
\begin{equation*}
\begin{split}
      \left| 1 -\frac{l_{F}(\theta)}{l_{S_{\mathrm{mean}}}(\theta)} \right| &= \left| \frac{\exp \left(-\frac{1}{2} \| y -m_{\mathrm{S,N}}^{\mathcal{G}}(\theta) \|_{\Gamma}^{2}\right) -\exp \left( -\frac{1}{2} \| y -\mathcal{G}_{\mathrm{X}}^\mathrm{F}(\theta) \|_{\Gamma}^{2}\right)}{\exp \left(-\frac{1}{2} \| y -m_{\mathrm{S,N}}^{\mathcal{G}}(\theta) \|_{\Gamma}^{2}\right)} \right|\\
      &\leq \left|\frac{1}{ \exp \left(-\frac{1}{2} \| y -m_{\mathrm{S,N}}^{\mathcal{G}}(\theta) \|_{\Gamma}^{2}\right) } \right| \left| \exp \left(-\frac{1}{2} \| y -m_{\mathrm{S,N}}^{\mathcal{G}}(\theta) \|_{\Gamma}^{2}\right) -\exp \left( -\frac{1}{2} \| y -\mathcal{G}_{\mathrm{X}}^\mathrm{F}(\theta) \|_{\Gamma}^{2}\right) \right|.
\end{split}
\end{equation*}
 Using the triangle inequality and the inequality $\|x\|_{\Gamma}^{2}\leq \lambda_{\text{max}}\|x\|_{2}^{2}$, where $\lambda_{\text{max}}$ is the maximum eigenvalue of $\Gamma^{-1}$, the first term of the last inequality can be bounded in the following way
\begin{equation*}
\begin{split}
    \left|\frac{1}{ \exp \left(-\frac{1}{2} \| y -m_{\mathrm{S,N}}^{\mathcal{G}}(\theta) \|_{\Gamma}^{2}\right) } \right|= \exp \left(\frac{1}{2} \| y -m_{\mathrm{S,N}}^{\mathcal{G}}(\theta) \|_{\Gamma}^{2}\right)
    &\leq \exp\left(\lambda_{\text{max}}( \| y \|^{2}_{2} + \| m_{\mathrm{S,N}}^{\mathcal{G}}(\theta) \|^{2}_{2})\right)\\
    &\leq\exp\left(\lambda_{\text{max}}( \| y \|^{2}_{2} +C_{\mathrm{S}}^2)\right),
\end{split}
\end{equation*}
where we have taken the supremum over $\theta$, and noting that $ \| m_{\mathrm{S,N}}^{\mathcal{G}}(\theta) \|^{2}_{2}=\big\| \mathbb{E}_{\mu^{\mathrm{S,N}}}[\mathcal{G}_{\mathrm{X}}^{\mathrm{S,N}}(\theta)] \big\|^{2}_{2}$, which using Jensen's inequality can be bounded by $ C_{\mathrm{S}}^2$ by assumption. Then, define $C_{B}=\exp\left(\lambda_{\text{max}}( \| y \|^{2}_{2} +C_{\mathrm{S}}^2) \right)$, and use the fact that the exponential function $\exp(-x)$ is locally Lipschitz on $[0,\infty)$, together with the triangle inequality, we obtain 
\begin{equation*}
\begin{split}
      \left| 1 -\frac{l_{F}(\theta)}{l_{S_{\mathrm{mean}}}(\theta)} \right|  &\leq \frac{1}{2}C_B\left| \| y -\mathcal{G}_{\mathrm{X}}^\mathrm{F}(\theta) \|_{\Gamma}^{2} -\| y -m_{\mathrm{S,N}}^{\mathcal{G}}(\theta) \|_{\Gamma}^{2}\right|\\
      &\leq \frac{1}{2}C_B\left| (\| y -m_{\mathrm{S,N}}^{\mathcal{G}}(\theta) \|_{\Gamma} + \| \mathcal{G}_{\mathrm{X}}^\mathrm{F}(\theta) -  m_{\mathrm{S,N}}^{\mathcal{G}}(\theta)\|_{\Gamma})^{2} -\| y -m_{\mathrm{S,N}}^{\mathcal{G}}(\theta) \|_{\Gamma}^{2}\right|\\
      &=\frac{1}{2}C_B(2\| y -m_{\mathrm{S,N}}^{\mathcal{G}}(\theta) \|_{\Gamma} \| \mathcal{G}_{\mathrm{X}}^\mathrm{F}(\theta) -  m_{\mathrm{S,N}}^{\mathcal{G}}(\theta)\|_{\Gamma} +  \| \mathcal{G}_{\mathrm{X}}^\mathrm{F}(\theta) -  m_{\mathrm{S,N}}^{\mathcal{G}}
      (\theta)\|_{\Gamma}^{2})\\
      &\leq \frac{1}{2}C_B(2\lambda_{\text{max}}\| y -m_{\mathrm{S,N}}^{\mathcal{G}}(\theta) \|_{2} \| \mathcal{G}_{\mathrm{X}}^\mathrm{F}(\theta) -  m_{\mathrm{S,N}}^{\mathcal{G}}(\theta)\|_{2} +  \lambda_{\text{max}}\| \mathcal{G}_{\mathrm{X}}^\mathrm{F}(\theta) -  m_{\mathrm{S,N}}^{\mathcal{G}}
      (\theta)\|_{2}^{2}).
\end{split}
\end{equation*}
Note that $\| y - m_{\mathrm{S,N}}^{\mathcal{G}}(\theta) \|_{2} \leq \| y \|_{2} + \| m_{\mathrm{S,N}}^{\mathcal{G}}(\theta) \|_{2}$ can be uniformly bounded using similar arguments as before.

We now need to take the expectation of the above bound with respect to $\pi^{y,F}$ (for $\theta$) and $\pi^{y,S}$ (for $\tilde{\theta}$). Note that for any random variable $X$, we have
\begin{align*}
\mathbb{E}_{\pi^{y,S}}\left[ X \right] &= \int_{\mathcal{T}} X \frac{1}{Z_{S}} \left (\exp \left(-\frac{1}{2} \| y- m_{\mathrm{S,N}}^{\mathcal{G}}(\theta)\|_{\Gamma}^{2} \right)\right) \pi_{0}(\theta) \\
&= \int_{\mathcal{T}} X \frac{Z_F}{Z_{S}} \frac{\left (\exp \left(-\frac{1}{2} \| y- m_{\mathrm{S,N}}^{\mathcal{G}}(\theta)\|_{\Gamma}^{2} \right)\right)}{\left (\exp \left(-\frac{1}{2} \| y- \mathcal{G}_{\mathrm{X}}^\mathrm{F}(\theta)\|_{\Gamma}^{2} \right)\right)} \pi^{y,F}(\theta) \\
&\leq C_\mathrm{S,F} \, \mathbb{E}_{\pi^{y,F}}\left[ X \right],
\end{align*}
where the constant $C_\mathrm{S,F} = \frac{Z_\mathrm{F}}{Z_\mathrm{S}} \exp\left( \frac{\lambda_\mathrm{max}}{2} (\|y\|^{2}_2 + C^{2}_\mathrm{F}) \right)$ is finite by assumption. Taking the expectation with respect to $\zeta$, and applying Minkowski's inequality and Hölder’s inequality with $p=\infty, q=1$, we obtain 
%
\begin{equation*}
\begin{split}
\mathbb{E}_{\zeta}\left[1-\alpha_\mathrm{mean}(\tilde{\theta}|\theta)\right] 
& \leq (1+C_\mathrm{S,F}) \lambda_\mathrm{max} C_B(\|y\|_2 + C_\mathrm{S}) \, \left(\mathbb{E}_{\pi^{y,F}} \left[ \| \mathcal{G}_{\mathrm{X}}^\mathrm{F} -  m_{\mathrm{S,N}}^{\mathcal{G}} \|_2 \right] + \mathbb{E}_{\pi^{y,F}} \left[ \| \mathcal{G}_{\mathrm{X}}^\mathrm{F} -  m_{\mathrm{S,N}}^{\mathcal{G}}\|_2^{2} \right] \right) \\
  & \leq C_1 \| \mathcal{G}_{\mathrm{X}}^\mathrm{F} -  m_{\mathrm{S,N}}^{\mathcal{G}} \|_{L_{\pi^{y,F}}^{2}(\mathcal T)} .
\end{split}
\end{equation*}
The proof for the marginal approximation is similar. We have
\begin{equation*}
\begin{split}
   &\left|1-\frac{l_{F}(\theta)}{l_{S_\mathrm{marginal}}(\theta)} \right|=\left|1-\frac{\exp\left(  -\frac{1}{2} \| y -\mathcal{G}_{\mathrm{X}}^\mathrm{F}(\theta) \|_{\Gamma}^{2} \right)}{\mathbb{E}_{\mu^{\mathrm{S,N}}}\big[ \exp\left( -\frac{1}{2} \| y -\mathcal{G}_{\mathrm{X}}^{\mathrm{S,N}}(\theta) \|_{\Gamma}^{2}\right)\big]} \right|\\
   &\leq \left|\frac{1}{ \exp\left(\mathbb{E}_{\mu^{\mathrm{S,N}}}\big[ -\frac{1}{2} \| y -\mathcal{G}_{\mathrm{X}}^{\mathrm{S,N}}(\theta) \|_{\Gamma}^{2}\big]\right) } \right|\mathbb{E}_{\mu^{\mathrm{S,N}}}\left[\left| \exp\left(-\frac{1}{2} \| y -\mathcal{G}_{\mathrm{X}}^{\mathrm{S,N}}(\theta) \|_{\Gamma}^{2}\right)- \exp\left(  -\frac{1}{2} \| y -\mathcal{G}_{\mathrm{X}}^\mathrm{F}(\theta) \|_{\Gamma}^{2}\right)\right|\right].
\end{split}
\end{equation*}
where Jensen’s inequality was applied to the exponential function as well as the absolute value. Note that the first term in the last inequality can be bounded again by the assumption, i.e.,
\begin{equation*}
    \left|\frac{1}{ \exp\left(\mathbb{E}_{\mu^{\mathrm{S,N}}}\big[ -\frac{1}{2} \| y -\mathcal{G}_{\mathrm{X}}^{\mathrm{S,N}}(\theta) \|_{\Gamma}^{2}\big]\right) } \right|= \exp\left(\mathbb{E}_{\mu^{\mathrm{S,N}}}\big[ \frac{1}{2} \| y -\mathcal{G}_{\mathrm{X}}^{\mathrm{S,N}}(\theta) \|_{\Gamma}^{2}\big]\right)
    \leq \exp\left(\lambda_{\text{max}}(\| y \|^{2}_{2} + C_{\mathrm{S}}^2)\right) = C_{\mathrm{B}},
\end{equation*}
and, by using the local Lipschitz continuity of the exponential function, we proceed as in the mean approximation to obtain,
\begin{equation*}
\begin{split}
      \left|1- \frac{l_{F}(\theta)}{l_{S_\mathrm{marginal}}(\theta)}\right|
      &\leq \frac{1}{2} C_B\mathbb{E}_{\mu^{\mathrm{S,N}}}\left[\left| \| y -\mathcal{G}_{\mathrm{X}}^\mathrm{F}(\theta) \|_{\Gamma}^{2} -\| y -\mathcal{G}_{\mathrm{X}}^{\mathrm{S,N}}(\theta) \|_{\Gamma}^{2}\right|\right]\\
      &\leq\frac{1}{2}C_B\mathbb{E}_{\mu^{\mathrm{S,N}}}\left[2\lambda_{\text{max}}\| y -\mathcal{G}_{\mathrm{X}}^{\mathrm{S,N}}(\theta) \|_{2} \| \mathcal{G}_{\mathrm{X}}^\mathrm{F}(\theta) -  \mathcal{G}_{\mathrm{X}}^{\mathrm{S,N}}(\theta)\|_{2} +  \lambda_{\text{max}}\| \mathcal{G}_{\mathrm{X}}^\mathrm{F}(\theta) -  \mathcal{G}_{\mathrm{X}}^{\mathrm{S,N}}(\theta)\|_{2}^{2}\right].
\end{split}
\end{equation*}
Applying Minkowski's inequality and Hölder's inequality with conjugate exponents $p = \frac{1+\delta}{\delta}$ and $q = 1+\delta$ for some $0 < \delta \leq 1$, to the first term in the last inequality, we have
\begin{equation*}
\begin{split}
      \left| 1-\frac{l_{F}(\theta)}{l_{S_\mathrm{marginal}}(\theta)}  \right|
      &\leq \lambda_{\text{max}}C_B\left(\left( \mathbb{E}_{\mu^{\mathrm{S,N}}}\left[ \| y -\mathcal{G}_{\mathrm{X}}^{\mathrm{S,N}}(\theta) \|_{2}^{(1+ \delta)/ \delta} \right] \right)^{\delta/ (1+ \delta)}\left(  \mathbb{E}_{\mu^{\mathrm{S,N}}}\left[ \| \mathcal{G}_{\mathrm{X}}^\mathrm{F}(\theta) -  \mathcal{G}_{\mathrm{X}}^{\mathrm{S,N}}(\theta)\|_{2}^{1+ \delta}\right] \right)^{1/(1+\delta)} \right.\\
      &+ \left. \mathbb{E}_{\mu^{\mathrm{S,N}}}\left[\| \mathcal{G}_{\mathrm{X}}^\mathrm{F}(\theta) -  \mathcal{G}_{\mathrm{X}}^{\mathrm{S,N}}(\theta)\|_{2}^{2})\right]\right).
\end{split}
\end{equation*}
Taking the expectation with respect to $\zeta$, and applying Minkowski's inequality and Hölder’s inequality with $p=\infty, q=1$, we obtain 
%
\begin{equation*}
\begin{split}
\mathbb{E}_{\zeta}\left[ 1-\alpha_{\mathrm{marginal}}(\tilde{\theta} | \theta) \right] & \leq (1+C_\mathrm{S,F}) \lambda_\mathrm{max}C_B (\|y\|_2 + C_\mathrm{S}) \, \left( \mathbb{E}_{\pi^{y,F}} \left[ \mathbb{E}_{\mu^{\mathrm{S,N}}} \left[ \| \mathcal{G}_{\mathrm{X}}^\mathrm{F}(\theta) -  \mathcal{G}_{\mathrm{X}}^{\mathrm{S,N}}(\theta)\|_{2}^{(1+ \delta)/ \delta} \right]^{\delta/ (1+ \delta)} \right] \right. \\
&+ \left. \mathbb{E}_{\pi^{y,F}} \left[ \mathbb{E}_{\mu^{\mathrm{S,N}}} \left[ \| \mathcal{G}_{\mathrm{X}}^\mathrm{F}(\theta) -  \mathcal{G}_{\mathrm{X}}^{\mathrm{S,N}}(\theta)\|_{2}^{2}\right] \right] \right)\\
&\leq C_{2} \left\| \left( \mathbb{E}_{\mu^{\mathrm{S,N}}}\left[ \| \mathcal{G}_{\mathrm{X}}^\mathrm{F} -  \mathcal{G}_{\mathrm{X}}^{\mathrm{S,N}}\|_{2}^{1+ \delta}\right] \right)^{1/(1+ \delta)} \right\|_{L_{\pi^{y,F}}^2(\mathcal T)}.
\end{split}
\end{equation*}
This completes the proof.
\end{pf}

In the particular case of a Gaussian process surrogate model, an application of \cref{prop:acceptance} and \cref{prop:gp_rate} gives the following convergence rates.

\begin{proposition}
\label{prop:gp_rates}
Suppose we use a standard Gaussian process surrogate model $\mathcal{G}_{\mathrm{X}}^\mathrm{{GP,N}}$ with $K(\theta,\theta') = k_p(\theta,\theta') I_{d_{y}}$, as defined in \eqref{eq:gp_pred}, 
and that the assumptions of  \cref{prop:gp_rate} are satisfied for $g = \mathcal G^\mathrm{F}_\mathrm{X}$.
Then there exist constants $C_{1}$, $C_{2}$, and $C_{3}$, independent of $N$, such that
\begin{equation*}
  \mathbb{E}_{\zeta}\left[ 1-\alpha_{\mathrm{mean}}(\tilde{\theta} | \theta) \right] 
\leq C_{1}h_{U}^{\tau}, 
\quad \text{and} \quad  
\mathbb{E}_{\zeta}\left[ 1-\alpha_{\mathrm{marginal}}(\tilde{\theta} | \theta) \right] 
\leq C_{2} h_{U}^{\tau} + C_{3}h_{U}^{\tau - \frac{d_{\theta}}{2}}.  
\end{equation*}
\end{proposition}

\begin{pf}
First, we note that in the case of a Gaussian surrogate model, where $\mathcal{G}_{\mathrm{X}}^{\mathrm{S,N}}(\theta) \sim \mathrm{GP}(m_{\mathrm{S,N}}^{\mathcal{G}}(\theta), \Gamma_{\mathrm{S,N}}(\theta,\theta'))$, the assumption in \cref{prop:acceptance} that $\sup_{\theta \in \mathcal{T}} \mathbb{E}_{\mu^{S,N}}\|\mathcal{G}_{\mathrm{X}}^{\mathrm{S,N}}(\theta)\|^2_{2} \leq C_{\mathrm{S}}^2$ has been proven to hold in \cite{aretha_posterior} using Fernique's theorem, where in particular it is shown that the constant $C_{\mathrm{S}}$ can be chosen independent of $N$. The other assumptions of \cref{prop:acceptance} are satisfied by \cref{prop:gp_rate}.
 
The first claim follows from \cref{prop:gp_rate} with $\beta=0$ and the first statement of \cref{prop:acceptance}, using that $\|\cdot\|_{L^2_{\pi^{y,F}}(\mathcal T)}$ can be bounded in terms of $\|\cdot\|_{L^2(\mathcal T)}$ as in the proof of \cref{prop:acceptance}. 
For the second claim, let us choose $\delta = 1$ in the second statement in \cref{prop:acceptance} and apply Jensen's inequality; then we have
\begin{equation*}
  \mathbb{E}_{\zeta}\left[ 1-\alpha_{\mathrm{marginal}}(\tilde{\theta} | \theta_{}) \right] 
\leq C_{2} \mathbb{E}_{\pi^{y,F}}\left[  \left( \mathbb{E}_{\mu^\mathrm{{GP,N}}}\left[ \| \mathcal{G}_{\mathrm{X}}^\mathrm{F} -  \mathcal{G}_{\mathrm{X}}^\mathrm{{GP,N}}\|_{2}^{2}\right] \right)^{1/2} \right] \leq  C_{2} \left( \mathbb{E}_{\pi^{y,F}}\left[  \mathbb{E}_{\mu^\mathrm{{GP,N}}}\left[ \| \mathcal{G}_{\mathrm{X}}^\mathrm{F} -  \mathcal{G}_{\mathrm{X}}^\mathrm{{GP,N}}\|_{2}^{2}\right]\right] \right)^{1/2} .  
\end{equation*}
Apply the triangle inequality gives
\begin{equation*}
\mathbb{E}_{\pi^{y,F}}\left[ \mathbb{E}_{\mu^\mathrm{{GP,N}}}\left[ \| \mathcal{G}_{\mathrm{X}}^\mathrm{F} -  \mathcal{G}_{\mathrm{X}}^\mathrm{{GP,N}}\|_{2}^{2}\right] \right]  \leq 2\mathbb{E}_{\pi^{y,F}}\left[ \mathbb{E}_{\mu^\mathrm{{GP,N}}}\left[ \| \mathcal{G}_{\mathrm{X}}^\mathrm{F} -  m_\mathrm{{GP,N}}^{\mathcal{G}}\|_{2}^{2}\right] \right] +2\mathbb{E}_{\pi^{y,F}}\left[ \mathbb{E}_{\mu^\mathrm{{GP,N}}}\left[ \| m_\mathrm{{GP,N}}^{\mathcal{G}} -  \mathcal{G}_{\mathrm{X}}^\mathrm{{GP,N}}\|_{2}^{2}\right] \right].
\end{equation*}
Using \cref{prop:gp_rate}, the first term can be bounded by
\begin{equation*}
   \mathbb{E}_{\pi^{y,F}}\left[ \mathbb{E}_{\mu^\mathrm{{GP,N}}}\left[ \| \mathcal{G}_{\mathrm{X}}^\mathrm{F} - m_\mathrm{{GP,N}}^{\mathcal{G}}\|_{2}^{2}\right] \right] = \mathbb{E}_{\pi^{y,F}}\left[ \| \mathcal{G}_{\mathrm{X}}^\mathrm{F} -  m_\mathrm{{GP,N}}^{\mathcal{G}}\|_{2}^{2}\right] \leq C_{2}^2 h_{U}^{2\tau}. 
\end{equation*}
For the second term,  we follow \cite{aretha_posterior} and use the equality $k_p(\theta,\theta) = \sup_{\|g\|_{H_{k}}=1}|g(\theta) - m_\mathrm{{GP,N}}^{\mathcal{G}}(\theta)|^{2}$, the Sobolev embedding theorem and \cref{prop:gp_rate} to obtain
\begin{equation*}
\begin{split}
\mathbb{E}_{\pi^{y,F}}\left[ \mathbb{E}_{\mu^\mathrm{{GP,N}}}\left[ \| m_\mathrm{{GP,N}}^{\mathcal{G}}(\theta) -  \mathcal{G}_{\mathrm{X}}^\mathrm{{GP,N}}(\theta)\|_{2}^{2}\right] \right] 
&= d_{y}\int_{\mathcal{T}} k_p(\theta,\theta)\pi^{y,F}(\theta)d\theta\\
&\leq \frac{d_{y}}{Z_{\mathrm{F}} } \sup_{\theta \in \mathcal{T}}\sup_{\|g\|_{H_{k}}=1}|g(\theta) - m_\mathrm{{GP,N}}^{\mathcal{G}}(\theta)|^{2}\\
&\leq  \frac{C_\mathrm{emb}d_{y}}{Z_{\mathrm{F}} }\sup_{\|g\|_{H_{k}}=1} \|g - m_\mathrm{{GP,N}}^{\mathcal{G}}\|^{2}_{H^{\frac{d_{\theta}}{2}}(\mathcal T)}\\
&\leq C_{3}^2 h_{U}^{2\tau - d_{\theta}}.
\end{split}
\end{equation*}
The result then follows from the inequality $\sqrt{x+y} \leq \sqrt{x} + \sqrt{y}$.
\end{pf}

For the case of neural network surrogates, we first note that convergence rates can be computed for the case of the mean approximation by using the result of \cref{prop:nn_rates}. In particular, if the operator $\mathcal{G}^\mathrm{F}_{\mathrm{X}}$ is sufficiently regular, then there exists a feedforward neural network $f_{a}(\cdot;\mathrm{W})$ satisfying
\begin{equation*}
    \mathbb{E}_{\zeta}\!\left[ 1 - \alpha_{\mathrm{mean}}(\tilde{\theta} \mid \theta_{j}) \right] 
\leq C_{1}\exp \left(-2^{-\frac{1}{2}}C_{2}M^{\frac{1}{2d_{\theta}+7}}\right),
\end{equation*}
where $M, C_1$ and $C_2$ are as in \cref{prop:nn_rates}. However, we emphasize here that the error bound obtained depends heavily on the neural network architecture, specifically the number of neurons per layer, the number of layers, and the use of the ReLU activation function. This construction does not fully align with the architecture used in our experiments to construct the neural network surrogate. Moreover, it remains unclear how to explicitly construct and train a neural network that attains these rates in practice.
In the case of the marginal approximation based on DeepGaLA, there is no theoretical background that could guide us toward establishing convergence rates. To the best of our knowledge, no existing theory provides insight into how the covariance matrix of DeepGaLA behaves or indeed converges to 0 as the size of the neural network increases.

\section{Numerical Experiments}\label{sec:5}

In this section, we study Bayesian inverse problems for three different differential equations: a one-dimensional boundary value problem, a two-dimensional elliptic PDE, and the Navier–Stokes equations in stream function formulation.  We start by briefly describing the numerical methods used to approximate the solutions to these problems and present the specific neural network architecture and training hyperparameters, along with details of the MCMC algorithm. The first two examples that we study are linear, and hence we compare our DeepGaLA approach with the PIGP, in terms of $\alpha_{\mathrm{val}}$ values, the bounds from \cref{prop:acceptance}, and the online evaluation time. In doing these comparisons we have decided to exclude the offline training time. The rationale behind this choice is that the two methods have fundamentally different training procedures. In particular,  for the PIGP model, training requires solving the differential equation in question for several values of $\theta$. Obtaining these solutions is problem-dependent: it can be relatively fast for simple cases, such as one-dimensional problems, but can become quite complex for high-dimensional problems or those defined on nonstandard domains. On the other hand, \cref{algo1} proposes a way to train a neural network; however, there are still many free parameters that affect training, such as the batch size, learning rate, and size of the training dataset. All these factors play an important role in the training cost.

\begin{table}[ht]
\centering
\begin{tabular}{lcc}
\toprule
\cmidrule(lr){1-3}
&$m_{\mathrm{S,N}}^{\mathcal{G}}$ & $\Gamma_{_\mathrm{{S,N}}}$\\
\midrule
$\mathrm{PIGP}$ & $2d_y(d_y + d_f + d_g)N + 3N$     & $(d_y + d_f + d_g)^3 N^3
+ 2 d_y^2 N(d_y + d_g + d_f)
+ 2 d_y^2$     \\
$\mathrm{dG}$  & $d_{y}(2 d_{L}d_n + d_L + (L-1)(5d_{n} + 2d_{n}^{2}) + d_{n}(6 d_{0}+ 9))$    &    $2d_{y}^{2}(d_n^2 + d_n)$ \\
\bottomrule
\end{tabular}
\caption{Online Cost of Evaluation for the PIGP and deepGaLA surrogate models}
\label{tab:cost}
\end{table}

Table \ref{tab:cost}   presents an online cost analysis of both the neural surrogate and the PIGP in terms of floating-point operations (FLOPs), considering the cost of evaluating the mean and variance at a new test input $\theta$. See \cref{appendix:costeval} for details of the derivations. As can be seen, the computational cost of evaluating the mean of the Gaussian process scales as $\mathcal{O}(N)$ with the number of training data points, while obtaining the variance scales as $\mathcal{O}(N^{3})$ due to the need to invert a matrix or solve a linear system, as in \eqref{eq:pigp_pred}. In contrast, the neural network exhibits a linear cost $\mathcal{O}(L)$ with respect to the number of layers and a quadratic cost $\mathcal{O}(d_n^{2})$ with respect to the number of neurons per hidden layer. The cost of obtaining the variance of DeepGALA requires evaluating the activations of the last hidden layer, which in turn requires evaluating the entire network. This can be obtained for free when evaluating the mean, as it is already computed during the forward pass in \texttt{PyTorch}. Therefore, the additional cost comes from matrix–vector multiplications and scales quadratically with the number of observed solutions $d_y$.

\subsection{1D Elliptic Boundary Balue Problem}
We consider  the following one-dimensional differential equation:
\begin{equation*}
\begin{split}
  -\frac{d}{dx}\left(\exp({\alpha(x,\theta)})\frac{du(x,\theta)}{dx} \right) = 4x&, \quad x\in D, \quad \theta \in \mathcal{T}, \\
  u(0,\theta) = 0, \quad u(1,\theta)=2&, \quad \text{on}  \quad \partial D
\end{split}
\end{equation*}
where   $D = (0,1)$. The coefficient $\alpha( \cdot,\theta) \in L^{\infty}(D)$ is defined by the expansion:
\begin{equation*}
     \alpha(x,\theta) = \sum_{n=1}^{d_{\theta}} \sqrt a_{n} \theta_{n} b_{n}(x),
\end{equation*}
where $a_n$ and $b_n$ correspond to the eigenvalues and eigenfunctions of the KL expansion of an exponential covariance kernel \cite{ghanem2003stochastic} and take the form $a_{n} = \frac{8}{w_{n}^{2} + 16}$ and $b_{n} = A_{n}\left(  \text{sin}(w_{n}x) + \frac{w_{n}}{4} \text{cos}(w_{n}x) \right)$, where $w_{n}$ is the $n_{th}$ solution of the equation $\text{tan} (w_{n}) = \frac{8w_{n}}{w_{n}^{2}-16}$ and $A_{n}$ is a normalization constant  such that $\|b_{n}\|_{L^2(D)} = 1$. 
We also assume that $\theta \in [-2,2]^{d_{\theta}} = \mathcal T$.

We now consider the inverse problem \eqref{inv_eq} where 
$
\mathcal{G}_{\mathrm{X}}(\theta)=\{u(x_{i},\theta) \}_{i=1}^{d_{y}}
$ and $\Gamma=\sigma^{2}I_{d_{y}}$, with $\sigma^{2}=10^{-4}$. Furthermore, we consider $d_{\theta} \in \{2, 3, 4, 5\}$ and $d_y = 6$ observations located within the domain $D$.  The observations were obtained by setting $\theta^{\dag} = [0.098, 0.430, 0.206, 0.090, -0.153]$ and using a FEM solver with 
$N_{\text{dof}}$ degrees of freedom and piece-wise linear basis functions. For the MCMC sampling, we used the RWMH algorithm with a uniform prior distribution 
$\mathrm{U}[-2,2]^{d_{\theta}}$ and drawn $2.5 \times 10^{6}$ samples from the posterior distributions. This number of samples was chosen so that the effective sample size is at least $10^{3}$, in particular for the problem in dimension 5. We consider the posterior distribution computed using the FEM, $\pi_{\text{FEM}}$, with $N_{\text{dof}} = 50$, as the ground 
truth. The DeepGaLA surrogate was constructed as follows: we used a neural network with one hidden layer, 60 neurons per layer, and Fourier embeddings with $d_{F} = 60$ and $\sigma_{FF} = 1$. The network was trained using \cref{algo1}, with 
$S = 5{,}000$ epochs, weights $\{ \lambda_i \}_{i=1} = 1$, and various sizes of the training set $\mathcal{D}$. 
The Laplace approximation was computed using the same training data. 

\begin{figure}[ht]
\centering
\includegraphics[width=\textwidth]{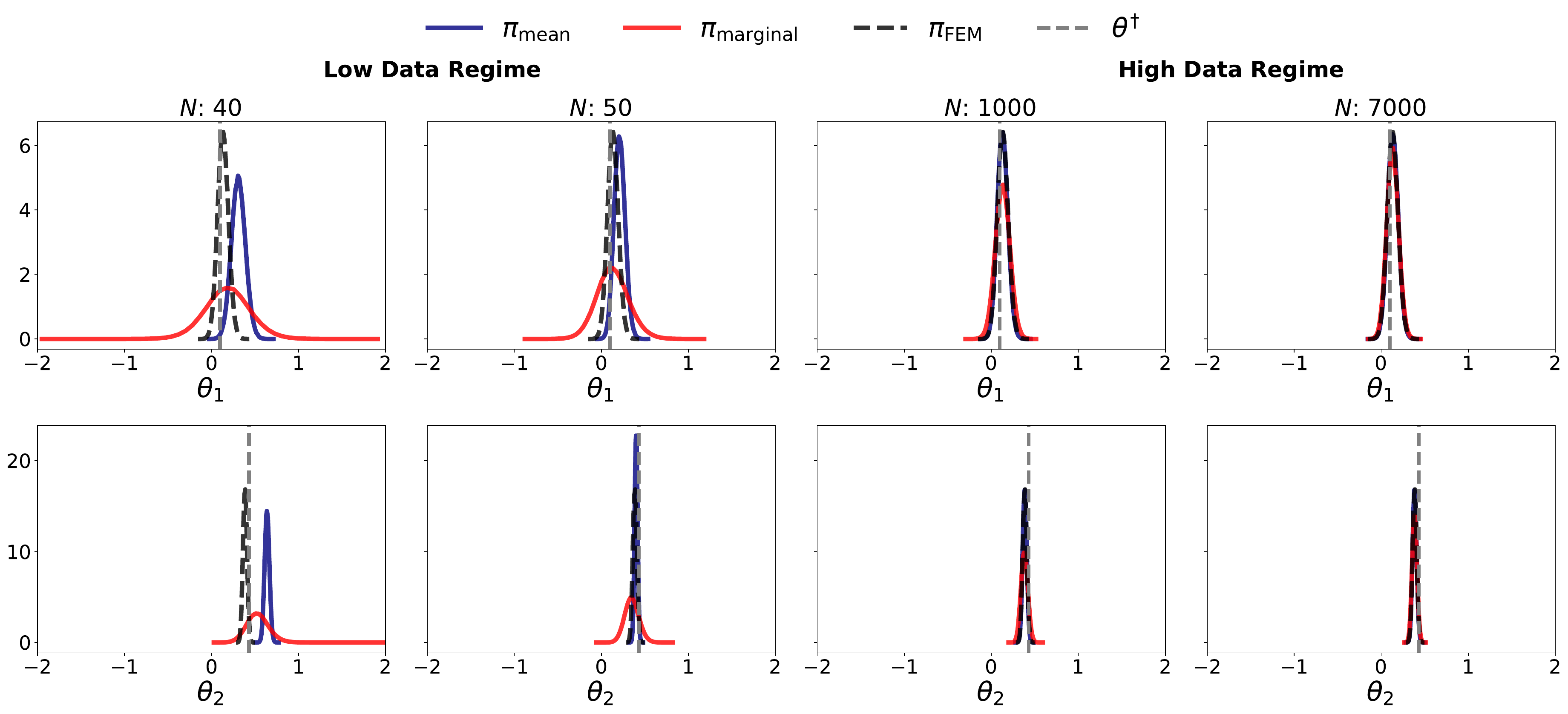}
\caption{1D BVP: Marginal ground-truth posterior distributions and the DeepGaLA mean and marginal approximations for different training set sizes.}
\label{bvp_fig:1}
\end{figure}

We now present in Figure \ref{bvp_fig:1} the mean and marginal posterior distributions obtained with DeepGaLA for different training sample sizes for $d_{\theta}=2$. As we can see in both cases, as the number of training points $\mathcal{D}$ increases, both the mean and the marginal posterior approximations converge to the ground truth posterior. However, the additional uncertainty of the marginal approximate posterior is beneficial in the low data regime, since it appropriately reflects the error in the approximation of the neural network, thus avoiding being overconfident around the wrong solution, as is the case with the mean approximate posterior. 

\begin{table}[ht]
\centering
\begin{tabular}{ccccccc}
\toprule
{$N$} & 
& 40 & 50 & 100 & 7{,}000 & 30{,}000\\
\midrule
\multirow{2}{*}{$\alpha_{\mathrm{val}}$ (\%)} 
& $\pi_{\mathrm{mean}}$ 
    & {1.45 $\pm$ 0.11}
    & {15.26$\pm$ 2.93}
    & {71.28$\pm$ 0.42}
    & {99.01$\pm$ 0.08}
    & {96.88$\pm$ 0.15} \\
& $\pi_{\mathrm{marginal}}$
    & {17.68$\pm$ 0.44}
    & {24.39$\pm$ 0.44}
    & {23.94$\pm$ 0.32}
    & {68.87$\pm$ 0.58}
    & {86.11$\pm$ 0.20} \\
\midrule

\multirow{2}{*}{$\|\cdot\|$}
& $\pi_{\mathrm{mean}}$
 &{$1.84 \times 10^{-1}$} 
 & {$3.28 \times 10^{-2}$} 
 & {$4.07 \times 10^{-3}$} 
 & {$3.41 \times 10^{-4}$} 
 & {$1.15 \times 10^{-3}$} \\
& $\pi_{\mathrm{marginal}}$
&{ $1.03 \times 10^{0}$ }
& {$1.99 \times 10^{0}$} 
& {$2.98 \times 10^{0}$} 
& {$5.05 \times 10^{-1}$} 
& {$7.50 \times 10^{-3}$} \\
\bottomrule
\end{tabular}
\caption{1D BVP: Mean $\pm$ standard deviation of $\alpha_{\mathrm{val}}$, together with the corresponding bounds for the DeepGaLA mean and marginal approximations, across different training set sizes for $d_{\theta}=2$.}
\label{bvp_table:1}
\end{table}

In Table \ref{bvp_table:1}, we present the values  of $\alpha_{\mathrm{val}}$, together with the error bounds established in \cref{prop:acceptance}, for $d_{\theta}=2$. We used a RW proposal to run \cref{algo2}, where the choice of the proposal variance $\beta$ was made carefully, since a poor choice can lead to inaccurate estimates of $\alpha_{\mathrm{val}}$ due to poor mixing of the DA chain. Therefore, we found that setting the proposal variance $\beta$ to the optimal value obtained for $\pi_{\mathrm{mean}}$, and using this same value to assess both approximate posteriors $\pi_{\mathrm{mean}}$ and $\pi_{\mathrm{marginal}}$, provides good mixing of the chains. Then,  $\alpha_{\mathrm{val}}$ was computed using Algorithm \ref{algo2} with $10^5$ evaluations of the FEM ground truth model, after which the samples were divided into groups of $N_F=10^{4}$ samples to obtain 10 realizations of $\alpha_{\mathrm{val}}$.  We provide the mean and standard deviation over these realizations in Table \ref{bvp_table:1}, and note that variations in the value of $\alpha_{\mathrm{val}}$ primarily come from a burn-in effect in \cref{algo2}.
The error bounds from \cref{prop:acceptance} were estimated by using $2.5\times 10^{5}$ samples from the FEM ground truth posterior, with the procedure repeated three times. We only report the mean value of the error bounds in \cref{bvp_table:1}, since the standard deviations were negligible.

The results in \cref{bvp_table:1} show that $\alpha_{\mathrm{val}}$ increases with the size of the training dataset while the corresponding error bounds become tighter, confirming the convergence of the mean and marginal approximations. In particular, $N = 7{,}000$ training samples are sufficient to obtain a posterior distribution close to the ground truth for the mean approximation. For $N = 30{,}000$, we observe a slight decrease in $\alpha_{\mathrm{val}}$ together with an increase in the error bound, which may be attributed to non-optimal training for this dataset size. Finally, consistent with results presented in Figure \ref{bvp_fig:1},  Table \ref{bvp_table:1} indicates that in the low-data regime, the marginal approximation achieves higher values of $\alpha_{\mathrm{val}}$ than the mean approximation, as it has a bigger overlap with the ground truth posterior. In our numerical experiments, although the actual values of $\alpha_{\mathrm{val}}$ may change, the general trends of $\alpha_{\mathrm{val}}$ increasing with $N$ and the marginal approximations having larger $\alpha_{\mathrm{val}}$ values than the mean approximations in the low-data regime, were consistently observed for varying values of $N_F$. Figure \ref{bvp_fig:2} further highlights that, in the low-data regime, the marginal approximation achieves a higher $\alpha_{\mathrm{val}}$ than the mean approximation across different parameter-space dimensions, highlighting the importance of incorporating uncertainty when data are limited or the parameter dimension is high. 

\begin{figure}[ht]
\centering
\includegraphics[width=1\textwidth]{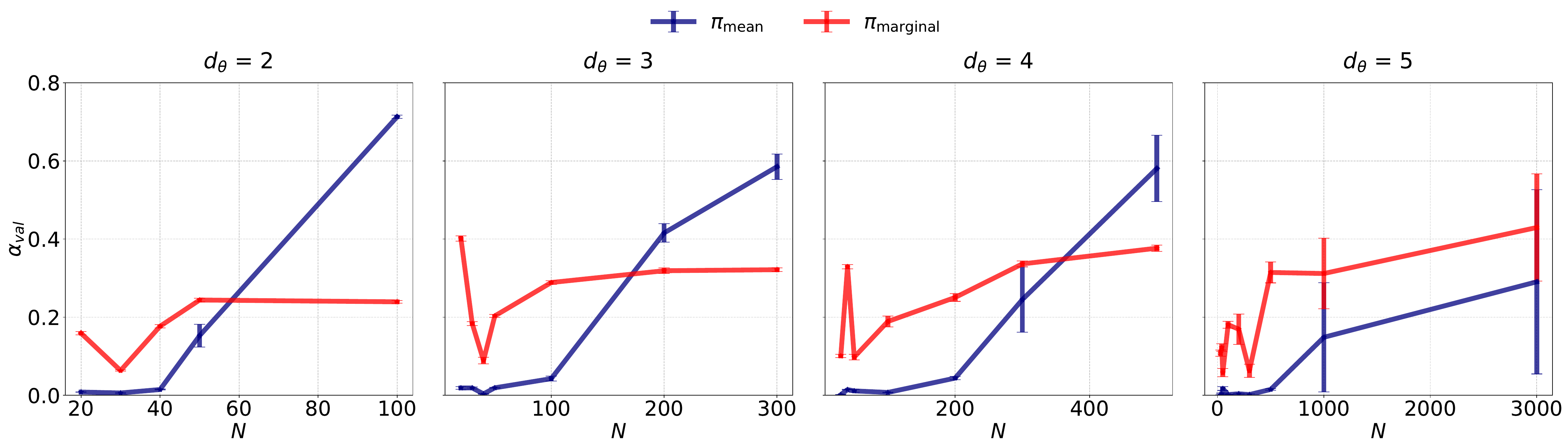}
\caption{1D BVP in the low-data regime: $\alpha_{\mathrm{val}}$ for the mean and marginal approximations across different values of $d_{\theta}$. }
\label{bvp_fig:2}
\end{figure}
\subsubsection{Comparison with PIGPs}

We now compare the  DeepGaLA surrogate with the PIGP surrogate in the high-data regime using the mean based approximate posterior. In the case of the PIGP surrogate, we have used a Matérn kernel of order $\frac{5}{2}$ for both the spatial and parameter spaces, while  $d_f = 10$, $d_g = 2$. We present the results of this comparison in Table \ref{tab:compa1}. As we can see the bounds in \cref{prop:acceptance} tighten as $\alpha_{\mathrm{val}}$ increases. Both surrogate models exhibit comparable evaluation times, with DeepGaLA being slightly faster. We do not report GPU times, as the evaluation of both surrogates at six points does not yield a significant difference; GPU acceleration would become advantageous in more computationally demanding scenarios.
Moreover, we can see that the evaluation time of the DeepGALA is essentially unaffected by the size of the training data; since the architecture remains nearly the same (aside from the number of inputs), its evaluation time remains stable as the dimension of the parameters grows. On the other hand, the evaluation time of the PIGP increases marginally with the size of the training data; however, the linear behavior expected from the cost analysis would become more apparent if a larger range of training points were considered.

\begin{table}[ht]
\centering
\begin{tabular}{ c c c c c c c}
\toprule
\multicolumn{6}{c}{1DBVP Inverse Problem: Mean Approximation} \\
\midrule
\multirow{3}{*}{$d_{\theta}$} & 
\multirow{3}{*}{S} & 
\multirow{3}{*}{$N$} &  
\multirow{3}{*}{$\|\cdot\|$} & 
\multirow{3}{*}{$\alpha_{\mathrm{val}}$(\%)} & 
{Mean } \\
 &  &  &  &  & Likelihood Eval.  \\
  &  &  &  &  & Time CPU (ms) \\
\midrule
\multirow[c]{2}{*}{2}  
 & $m^{\mathcal{G}}_{PIGP}$ & 75 
& $3.41 \times 10^{-4}$ 
& $99.20 \pm 0.08$ 
& $0.23 \pm 0.01$  \\
 & $m^{\mathcal{G}}_{dG}$ & 7,000 
& $3.02 \times 10^{-4}$ 
& $99.01 \pm 0.26$ 
& $0.22 \pm 0.03$ \\
\midrule
\multirow[c]{2}{*}{3} 
 & $m^{\mathcal{G}}_{PIGP}$  & 100 
& $7.47 \times 10^{-4}$ 
& $96.80 \pm 0.18$ 
& $0.24 \pm 0.03$  \\
 & $m^{\mathcal{G}}_{dG}$ & 5,000
& $9.08 \times 10^{-4}$ 
& $95.88 \pm 0.19$ 
& $0.22 \pm 0.01$ \\
\midrule
\multirow[c]{2}{*}{4} 
 & $m^{\mathcal{G}}_{PIGP}$  & 160 
& $1.48  \times 10^{-3}$ 
& $93.47 \pm 0.47$ 
& $0.27 \pm 0.02$ \\
 & $m^{\mathcal{G}}_{dG}$ &  9,000 
& $1.18 \times 10^{-3} $ 
& $94.61 \pm 0.65$ 
& $0.22 \pm 0.01$ \\
\midrule
\multirow[c]{2}{*}{5} 
 & $m^{\mathcal{G}}_{PIGP}$  & 200 
& $4.63 \times 10^{-3}$ 
& $87.22 \pm 8.66$ 
& $0.26 \pm 0.02$ \\
 & $m^{\mathcal{G}}_{dG}$ & 9,000
& $3.47 \times 10^{-3}$ 
& $87.70 \pm 4.24$ 
& $0.22 \pm 0.01$ \\
\bottomrule
\end{tabular}
\caption{1D BVP: Mean $\pm$ standard deviation of $\alpha_{\mathrm{val}}$, together with the corresponding bounds for the DeepGaLA mean and marginal approximations, across different training set sizes for $d_{\theta}=2$.}
\label{tab:compa1}
\end{table}
 
We additionally present in Figure \ref{bvp_fig:3}, the evaluation times of the marginal approximation using DeepGaLA and PIGP. As we can see, the evaluation time using DeepGaLA remains largely unchanged across different values of $\alpha_{\mathrm{val}}$ as the dimensionality of the parameter space increases. In contrast, the evaluation cost for the PIGP grows with increasing dimension, consistent with the dimension-dependent convergence rates.

\begin{figure}[ht]
    \centering
\includegraphics[width=1\textwidth]{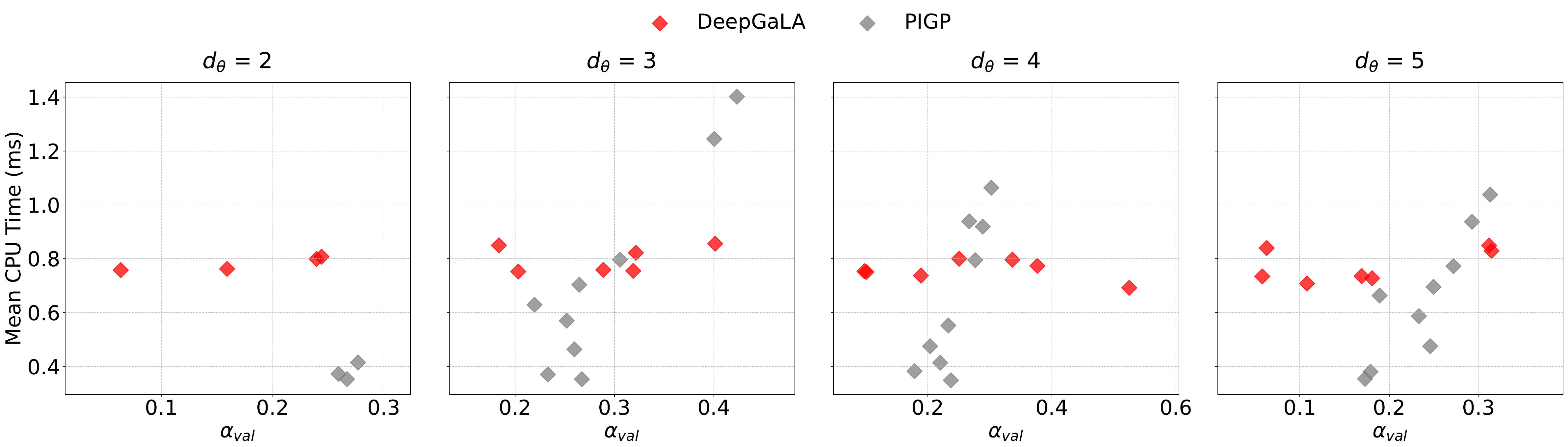}
\caption{1D BVP: Evaluation time in the low-data regime for PIGP and DeepGaLA across different parameter-space dimensions. Each data point corresponds to a different choice of training data size.}
\label{bvp_fig:3}
\end{figure}


\subsection{2D-Elliptic Inverse Problem }

Consider now a two-dimensional elliptic partial differential equation given by:
\begin{equation*}
\begin{split}
    -\nabla \cdot \left( \exp(\alpha(x,\theta))\nabla u(x,\theta) \right) &= 4x_{1}x_{2}, \quad x\in D,\quad \theta \in \mathcal{T},\\
    u(x,\theta) &= 0, \quad \text{on} \quad \partial D
\end{split}
\end{equation*}
where $D = (0,1)^{2}$. The coefficient $\alpha( \cdot,\theta) \in L^{\infty}(D)$ is expanded in the following way:
\begin{equation*}
    \alpha(x,\theta) =  \sum_{n=1}^{d_{\theta}} 
    \sqrt{a_{n}^{2D}}  
    \theta_{n} b_{n}^{2D}(x)  ,
\end{equation*}
where $a_{i}^{2D}=a_{j_{n}}a_{i_{n}} $ and $b_{i}^{2D}(x) = b_{j_{n}}(x_{1}) b_{i_{n}}(x_{2})$ for some $j_{n},i_{n} \in \mathbb{N}$ are obtained as products of the eigenvalues, $a_{j_{n}}$, and eigenfunctions, $b_{j_{n}}$, of the one–dimensional problem such that  $a_{1}>a_{2 }> ,..,>a_{d_{\theta}}$. We assume that the parameters $\theta$ take values in the parameter space $\mathcal{T} = [-4,4]^{d_{\theta}}$. The setup of the inverse problem \eqref{inv_eq} is similar to the previous experiment. We consider once more $\mathcal{G}_{\mathrm{X}}(\theta) = \{ u(x_i,\theta) \}_{i=1}^{d_y}$, $\Gamma = \sigma^{2}I_{d_{y}}$ with $\sigma^{2} = 10^{-4}$, $d_{\theta} \in \{2,3,4,5\}$ and $d_{y} = 6$ observations located inside the domain $D$. The observations are obtained using the FEM on a $75 \times 75$ grid and piece-wise linear basis functions. The MCMC setup follows the configuration of the first experiment, using a uniform prior $U[-4,4]^{d_{\theta}}$ and drawing $2.5 \times 10^{6}$ samples from the posterior distributions. We consider the posterior distribution computed using the FEM, $\pi_{\mathrm{FEM}}$, on a $50 \times 50$ grid as the ground truth.

For this experiment, the DeepGaLA surrogate was constructed with the following setup: we used one hidden layer with 80 neurons and Fourier embeddings with $d_{F} = 80$ and $\sigma_{FF} = 1$. The neural network was trained using \cref{algo1} over $S= 5,000$ epochs with adaptive weighting and various sizes of $\mathcal{D}$. The Laplace approximation was fitted using the same datasets. The mean and marginal posterior distributions obtained by DeepGaLA are presented in Figure \ref{ell2d_fig:1}. This experiment shows results similar to the previous one. Both the mean and marginal approximations converge to the ground truth as the number of training points in $\mathcal{D}$ increases. Additionally, we can again see the benefit of incorporating uncertainty in the neural network in low-data regimes, which helps prevent overconfidence when the predictions diverge from the truth.

\begin{figure}[ht]
\centering
\includegraphics[width=\textwidth]{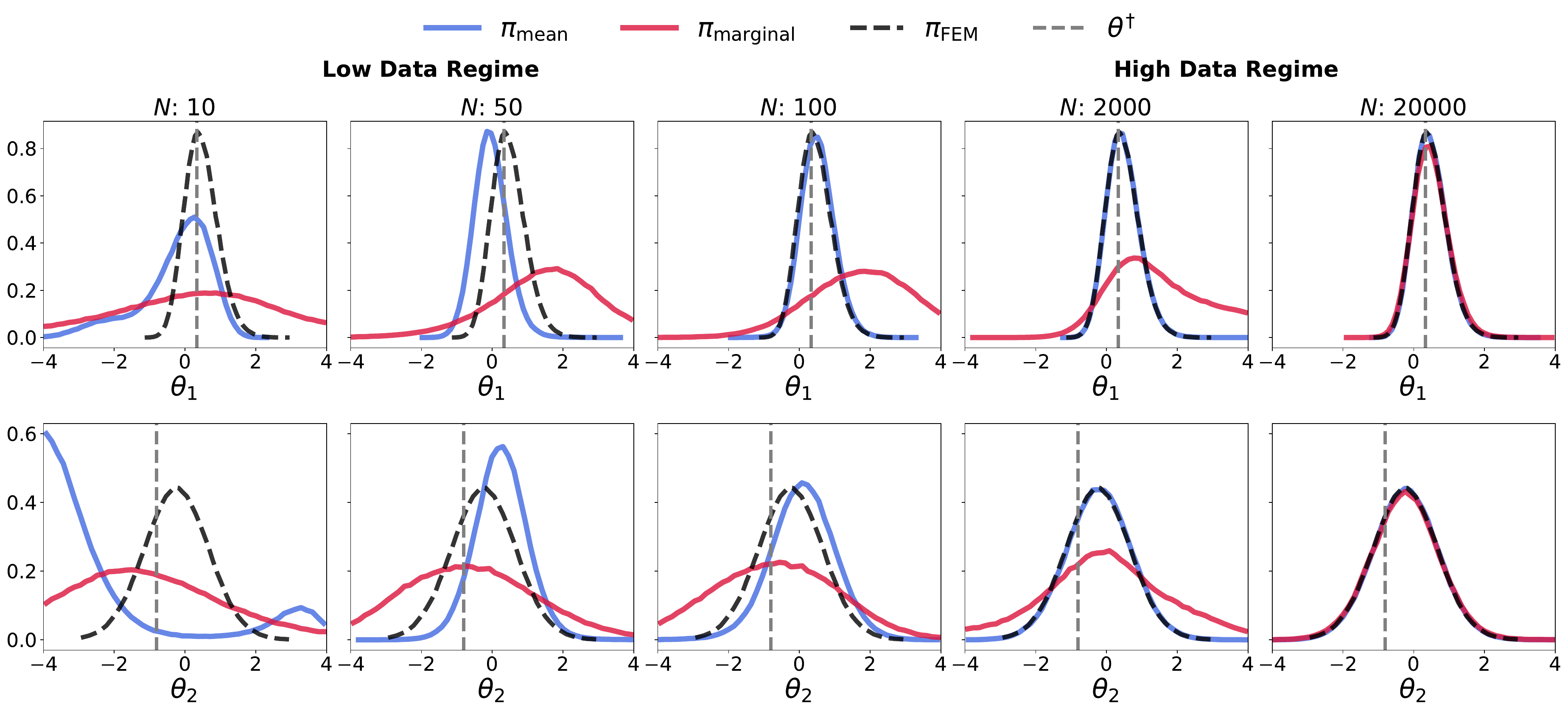}
\caption{Elliptic-2D: Marginal ground-truth posterior distributions and the DeepGaLA mean and marginal approximations for different training set sizes.}
\label{ell2d_fig:1}
\end{figure}

In Table \ref{ell2d_tab:1}, we present the mean values and the standard deviation of $\alpha_{\mathrm{val}}$ and the  mean error bounds, where we have omitted the standard deviations due to their negligible magnitude.  Again, we computed these quantities following the same procedure as in the previous experiment. In particular, $\alpha_{\mathrm{val}}$ was computed using \cref{algo2} by evaluating the finer model $10^{5}$ times, and then partitioning these evaluations into 10 realizations of $\alpha_{\mathrm{val}}$. The error bounds were estimated by using $2.5\times 10^{5}$ samples from the true posterior; this procedure was repeated three times. Moreover, we again used a RW proposal for \cref{algo2} and set the proposal variance $\beta$ to the value that is optimal for $\pi_{\mathrm{mean}}$. This same value was then used to assess both $\pi_{\mathrm{mean}}$ and $\pi_{\mathrm{marginal}}$. As expected, $\alpha_{\mathrm{val}}$ increases and the error bounds decrease as the mean approximation approaches the true posterior. Once again, Table \ref{ell2d_tab:1} shows consistent results with Figure \ref{ell2d_fig:1}, confirming that the marginal approximation is beneficial in the low-data regime, where its greater overlap with the true posterior results in a higher $\alpha_{\mathrm{val}}$. Figure \ref{ell2d_fig:2} further confirms the importance of the marginal approximation in the low-data regime, as $\alpha_{\mathrm{val}}$ is consistently larger for the marginal approximation than for the mean approximation across the parameter-space dimensions shown in the figure.

\begin{table}[ht]
\centering
\begin{tabular}{ccccccc}
\toprule
{$N$} & 
& 10 & 50 & 100 & 2{,}000 & 20{,}000 \\
\midrule
\multirow{2}{*}{$\alpha_{\mathrm{val}}$ (\%)} 
& $\pi_{\mathrm{mean}}$ 
    & $15.79 \pm 0.62$
    & $30.74 \pm 0.63$
    & $76.74 \pm 0.59$
    & $98.89 \pm 0.11$
    & $99.21 \pm 0.09$ \\
& $\pi_{\mathrm{marginal}}$
    & $22.92 \pm 0.37$
    & $33.64 \pm 0.54$
    & $37.62 \pm 0.44$
    & $51.06 \pm 0.57$
    & $95.68 \pm 0.11$ \\
\midrule
\multirow{2}{*}{$\|\cdot\| $}
& $\pi_{\mathrm{mean}}$
    & $5.15  \times 10^{-2}$
    & $5.23 \times 10^{-2}$
    & $4.66 \times 10^{-2}$
    & $4.59 \times 10^{-2}$
    & $4.60 \times 10^{-2}$ \\
& $\pi_{\mathrm{marginal}}$
    & $5.46 \times 10^{-1}$
    & $1.37\times 10^{-1} $
    & $1.21 \times 10^{-1}$
    & $4.48 \times 10^{-2}$
    & $3.29 \times 10^{-2}$ \\
    \bottomrule
\end{tabular}
\caption{Elliptic-2D: Mean $\pm$ standard deviation  of $\alpha_{\mathrm{val}}$ and corresponding bounds for the DeepGaLA mean and marginal approximations across different training set sizes.}
\label{ell2d_tab:1}
\end{table}

\begin{figure}[ht]
\centering
\includegraphics[width=1\textwidth]{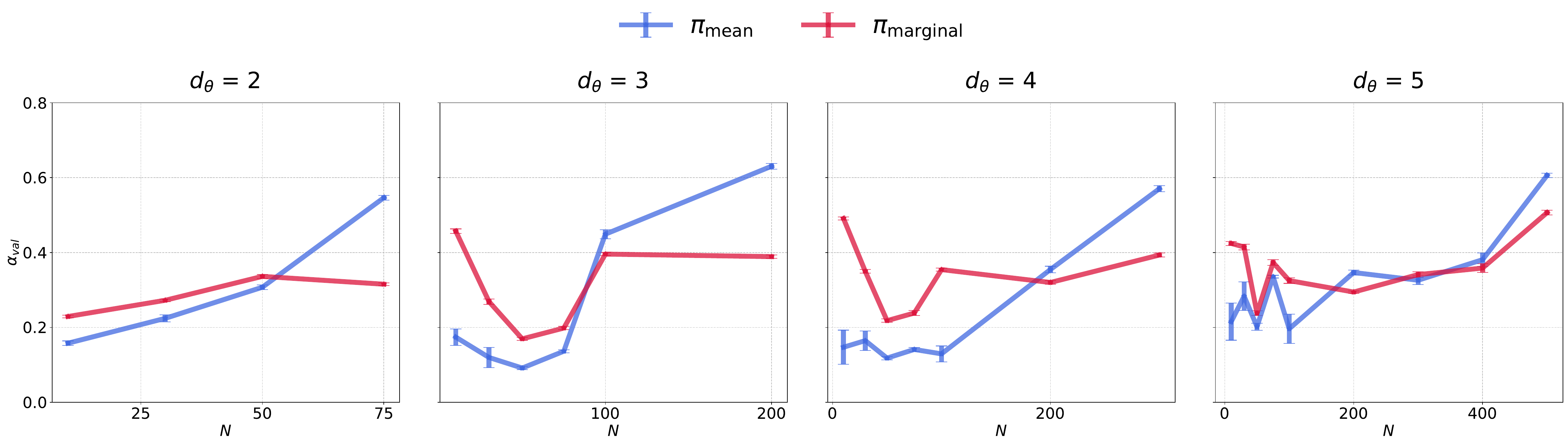}
\caption{Elliptic-2D: $\alpha_{\mathrm{val}}$ for the mean and marginal approximations across different values of $d_{\theta}$.}
\label{ell2d_fig:2}
\end{figure}

\subsubsection{Comparison with PIGPs}

We next perform a comparison between the DeepGaLA and PIGP surrogate in the high-data regime based on the mean approximate posterior.  For the PIGP surrogate, we  employed a Matérn kernel of order $\tfrac{5}{2}$ for both the spatial and parameter spaces, with $d_f = 10$, $d_g = 20$. Table \ref{ell2d_tab:2} shows the results of this comparison. As we can see, the bounds across the surrogate models are very similar, and, as expected, the corresponding values of $\alpha_{\mathrm{val}}$ are also close. As in the previous experiment, DeepGaLA exhibits slightly faster evaluation times. This could be further improved by identifying a smaller architecture, once again highlighting the computational advantage of the deep surrogate.

\begin{table}[ht]
\centering
\begin{tabular}{ c c c c c c }
\toprule
\multicolumn{6}{c}{Elliptic Inverse Problem: Mean Approximation} \\
\midrule
\multirow{3}{*}{$d_{\theta}$} & 
\multirow{3}{*}{S} & 
\multirow{3}{*}{$N$} &  
\multirow{3}{*}{$\|\cdot\|_{L_{\mu^{y}}}$} & 
\multirow{3}{*}{$\alpha_{\mathrm{val}}$(\%)} & 
{Mean } \\
 &  &  &  &  & Likelihood Eval.  \\
  &  &  &  &  & Time CPU (ms) \\
\midrule
\multirow[c]{2}{*}{2}  
 & $m^{\mathcal{G}}_{\mathrm{PIGP}}$ & 150 
 & $4.60 \times 10^{-2}$ 
 & $98.76 \pm 0.10$ 
 & $0.27 \pm 0.01$\\
 & $m^{\mathcal{G}}_{\mathrm{dG}}$ & 2,000 
 & $4.59 \times 10^{-2}$ 
 & $98.89 \pm 0.11$ 
 & $0.23 \pm 0.02$\\
\midrule
\multirow[c]{2}{*}{3} 
 & $m^{\mathcal{G}}_{\mathrm{PIGP}}$  & 150 
 & $3.67 \times 10^{-2}$ 
 & $96.32 \pm 0.13$
 & $0.27 \pm 0.03$\\
 & $m^{\mathcal{G}}_{\mathrm{dG}}$ & 2,000 
 & $3.65 \times 10^{-2}$ 
 & $95.95 \pm 0.21$ 
 & $0.23 \pm 0.01$\\
\midrule
\multirow[c]{2}{*}{4} 
 & $m^{\mathcal{G}}_{\mathrm{PIGP}}$  & 150 
 & $3.44 \times 10^{-2}$ 
 & $94.95 \pm 0.30$ 
 & $0.26 \pm 0.02$\\
 & $m^{\mathcal{G}}_{\mathrm{dG}}$ &  2,000 
 & $3.44 \times 10^{-2}$ 
 & $94.98 \pm 0.23$ 
 & $0.23 \pm 0.01$\\
\midrule
\multirow[c]{2}{*}{5} 
 & $m^{\mathcal{G}}_{\mathrm{PIGP}}$  & 170 
 & $3.09 \times 10^{-2}$ 
 & $85.88 \pm 0.56$
 & $0.29 \pm 0.01$\\
 & $m^{\mathcal{G}}_{\mathrm{dG}}$ & 5,000
 & $3.10 \times 10^{-2}$ 
 & $85.19 \pm 0.43$
 & $0.23 \pm 0.01$\\
\bottomrule
\end{tabular}
\caption{Elliptic-2D: Comparison of PIGP and DeepGaLA across multiple metrics and parameter-space dimensions.}
\label{ell2d_tab:2}
\end{table}

Finally, Figure \ref{ell2d_fig:3} shows the evaluation time of the marginal approximation as a function of $\alpha_{\mathrm{val}}$ for both DeepGaLA and PIGP. For this particular experiment, the PIGP surrogate achieves faster evaluation times for comparable values of $\alpha_{\mathrm{val}}$. Nevertheless, as the dimensionality of the parameter space increases, the evaluation time required by PIGP to attain reasonable values of $\alpha_{\mathrm{val}}$ grows significantly. This behavior is not observed for DeepGaLA, whose evaluation time remains essentially constant, as it is driven by the network architecture rather than by the size of the training data.

\begin{figure}[ht]
    \centering
\includegraphics[width=1\textwidth]{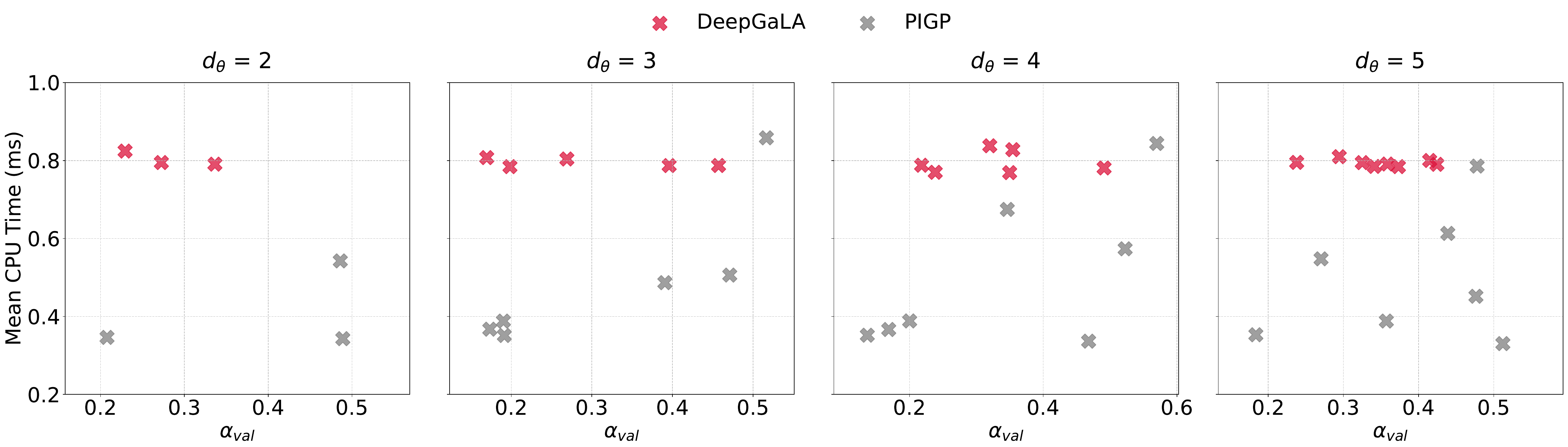}
\caption{Elliptic-2D:  Evaluation time in the low-data regime for PIGP and DeepGaLA across different parameter-space dimensions. Each data point corresponds to a different choice of training data size.}
\label{ell2d_fig:3}
\end{figure}

\subsection{Navier-Stokes Inverse Problem}
Let us now consider the vorticity-stream formulation of the Navier-Stokes equations given the following equation:
\begin{equation*}\label{eq:nv} 
\begin{split} 
 \frac{\partial \omega(z,t,\theta)}{\partial t} + U(z,t,\theta) \cdot \nabla \omega(z,t,\theta) &= \nu \nabla^2 \omega(z,t,\theta) + f(z), \quad (z,t) \in D, \quad  \theta \in \mathcal{T},\\
 \nabla^2 \psi(z,t,\theta) &= -\omega(z,t,\theta), \quad (z,t) \in D, \\
 w(z, 0,\theta) &= w_0(z,\theta), \quad z \in D_{s},
\end{split} 
\end{equation*}
where $D = D_{s} \times (0, T)$, $D_{s}=[0, 2\pi]^{2}$ and  $T \in \mathbb{R}^{+}$. The functions $\omega$ and $\psi$ are known as the vorticity and stream function, respectively, and $f(z) = 0.001(\sin(2\pi(z_1 + z_2)) + \cos(2\pi(z_1 + z_2)))$ represents the vorticity form of the forcing term. The term $U(z,t,\theta) = (u(z,t,\theta), v(z,t,\theta))$ is the velocity, with components derived from the stream function as, $u(z,t,\theta) = \partial \psi(z,t,\theta) /\partial z_{2}, \quad v = -\partial \psi(z,t,\theta)/\partial z_{1}$.
The initial condition is given by a random field $\mathcal{N}(0, 7^{3/2} (- \Delta + \tau^{2} I)^{-\alpha/2})$, with periodic boundary conditions. Under this setup, the random field has a Karhunen-Loève expansion given by:
\begin{equation*} \label{eq:kl} w_0(z,\theta) = 7^{3/2} \sum_{n=1}^{\frac{d_{\theta}}{2}} \frac{1}{\mu_n^{\alpha / 2}} \left( \phi_{1,n} (z)\theta_{1,n}, + \phi_{2,n} (z)\theta_{2,n} \right), 
\end{equation*}
where $\mu_n = \tau^{2} + n_{z_1}^{2} + n_{z_2}^{2}$ are the eigenvalues  and $\phi_{1,n}(z) = \cos(n_{z_1} z_1 + n_{z_2} z_2)/\sqrt{2\pi}$ and $\phi_{2,n}(z) = \sin(n_{z_1} z_1 + n_{z_2} z_2)/\sqrt{2\pi}$ are the eigenfunctions of the shifted Laplacian $- \Delta + \tau^{2} I$, for $n_{z_1}, n_{z_2} \in \mathbb{Z}$. We consider $\theta \in [-2,2]^{d_{\theta}}$, where $d_{\theta}=2m$ for some $m \geq 1$ so that the parameters are arranged in pairs.

The inverse problem \eqref{inv_eq} is set up such that $\mathcal{G}_{\mathrm{X}}(\theta) = \{ u(z_i, T,\theta) \}_{i=1}^{d_y}$, and $\Gamma = \sigma^{2}I_{d_{y}}$, with $\sigma^{2} = 10^{-2}$. We consider $d_{\theta} = 2$ and $d_y = 6$ observations  in the domain $D_{s}$ at time $T = 2$. The observations were obtained using a pseudo-spectral method combined with the Crank–Nicolson scheme for time integration, with parameters $\nu = 10^{-2}$, $\tau = \sqrt{2}$, and $\alpha = 5$.  For the MCMC, we used the RWMH algorithm with a uniform prior $\mathrm{U}[-2,2]^{d_\theta}$ and drew $2.5 \times 10^6$ samples from the posterior distribution. The DeepGaLA surrogate was constructed as follows: 2 hidden layers with 300 neurons per layer, a Fourier embedding with embedding scale set to 1, $\sigma_{\mathrm{FF}}^2 = 1$. A periodic embedding was used to enforce periodic boundary conditions. The neural network was trained using \cref{algo1} with different sizes of $\mathcal{D}$ and $S = 5,000$ epochs, updating the weights $\{ \lambda_i \}_{i=1}^{N}$ every 250 epochs. To fit the Laplace approximation, we used the same training data.

We now present in Figure \ref{nv_fig:1} the mean and marginal approximate posterior distributions obtained with DeepGaLA, together with the true values of $\theta$. We observe that $\pi_{\mathrm{marginal}}$ converges toward $\pi_{\mathrm{mean}}$ as the size of the training dataset $\mathcal{D}$ increases. At the same time, $\pi_{\mathrm{mean}}$ becomes increasingly concentrated around the true parameter values as $\mathcal{D}$ grows. Nevertheless, without access to the true posterior, it is not possible to assess how close these posterior approximations are to the ground truth; in this case, computing the true posterior directly was computationally prohibitive. We employ \cref{algo2} to assess this, using the previously described pseudo-spectral method as the fine model and a RW proposal with a proposal variance chosen as in the previous experiments. The fine model was evaluated $5 \times 10^{3}$ times and then the resulting evaluations were partitioned into 10 realizations of $\alpha_{\mathrm{val}}$. The mean values of $\alpha_{\mathrm{val}}$, along with their standard deviations, are reported in Table \ref{nv_tab:1}. From these results, we can conclude that the mean approximation converges to the ground truth posterior as the amount of training data increases. In contrast, and consistent with our previous findings, the marginal approximation plays a crucial role in the low-data regime by mitigating model overconfidence, which is evident in the mean approximation being far from the true posterior in this setting.

\begin{figure}[ht]
\centering
\includegraphics[width=\textwidth]{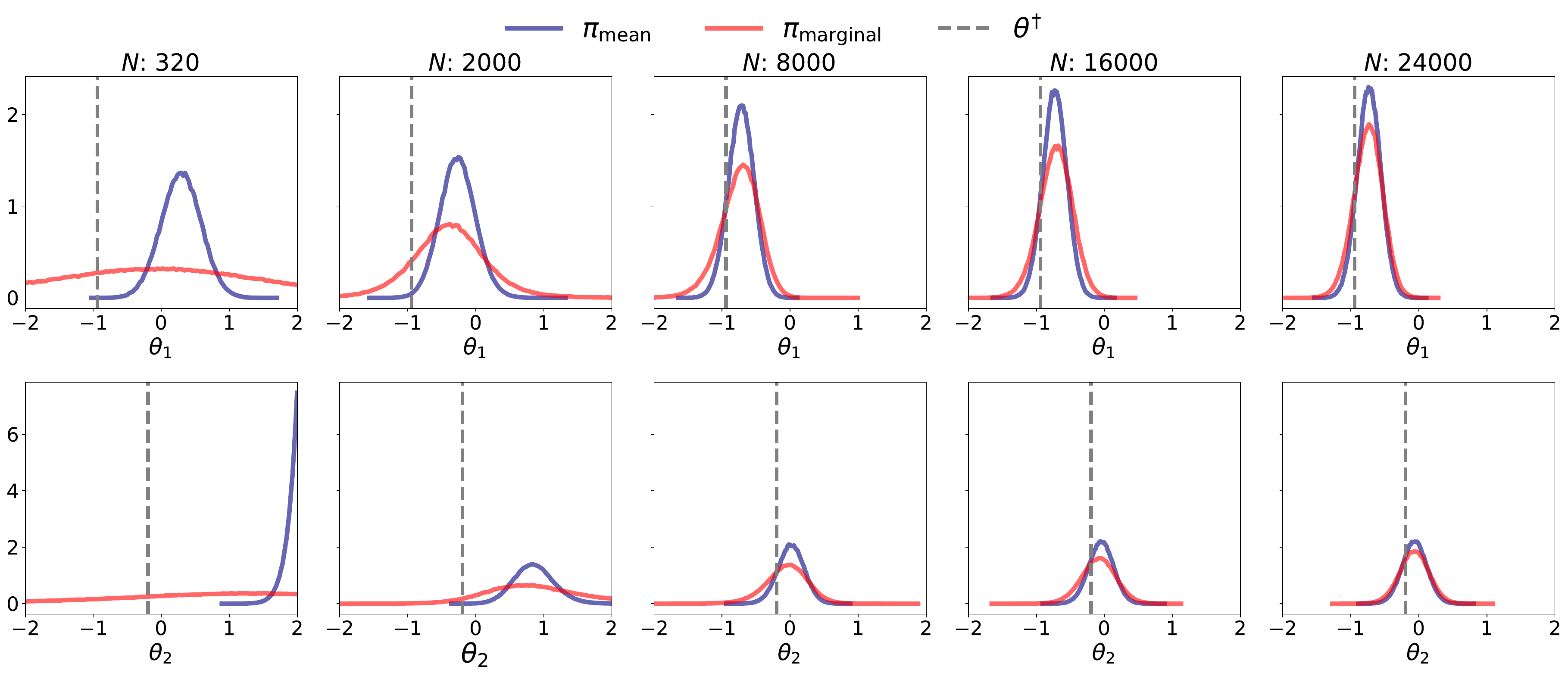}
\caption{NV: Mean and marginal approximations of the posterior distributions using DeepGaLA for different training set sizes.}
\label{nv_fig:1}
\end{figure}

\begin{table}[ht]

\centering
\begin{tabular}{ccccccc}
\toprule
{$N$} & 
& 320 & 2{,}000  & 8{,}000  & 16{,}000 & 24{,}000 \\

\midrule
\multirow{2}{*}{$\alpha_{\mathrm{val}}$ (\%)} 
& $\pi_{\mathrm{mean}}$ 
    & $12.66 \pm 1.94$
    & $11.28 \pm 1.11$
    & $67.26 \pm 1.40$
    & $75.10 \pm 1.55$
    & $75.04 \pm 2.10$ \\
& $\pi_{\mathrm{marginal}}$
    & $33.68 \pm 2.38$
    & $20.30 \pm 2.17$
    & $57.02 \pm 2.31$
    & $62.50 \pm 2.60$
    & $73.72 \pm 0.94$ \\
\bottomrule
\end{tabular}
\caption{NV: Mean $\pm$ standard deviation of $\alpha_{\mathrm{val}}$ for the DeepGaLA mean and marginal approximations across different training set sizes.}
\label{nv_tab:1}
\end{table}

Finally, Figure \ref{nv_fig:3} shows the DeepGaLA predictions evaluated at the mode of $\pi_{\mathrm{mean}}$. In the low-data regime, this mode differs significantly from the ground truth. In contrast, as more training data are used, the mode of $\pi_{\mathrm{mean}}$ moves closer to the ground truth. Figure \ref{nv_fig:4} displays the absolute error of the DeepGaLA prediction with respect to the ground truth, together with the model uncertainty. In the low-data regime, the error and the uncertainty are of similar magnitude. In the high-data regime, this relationship is still observed, but at an overall smaller scale. Moreover, although the alignment is not perfect, the DeepGaLA uncertainty tends to concentrate in regions where the prediction error is large, in both the low- and high-data regimes. The remaining mismatch can be attributed to the fact that DeepGaLA treats only the last layer as stochastic. Using a fully Bayesian neural network could potentially improve the uncertainty estimates, but at a significantly higher computational cost.

\begin{figure}[ht]
\centering
\includegraphics[width=0.8\textwidth]{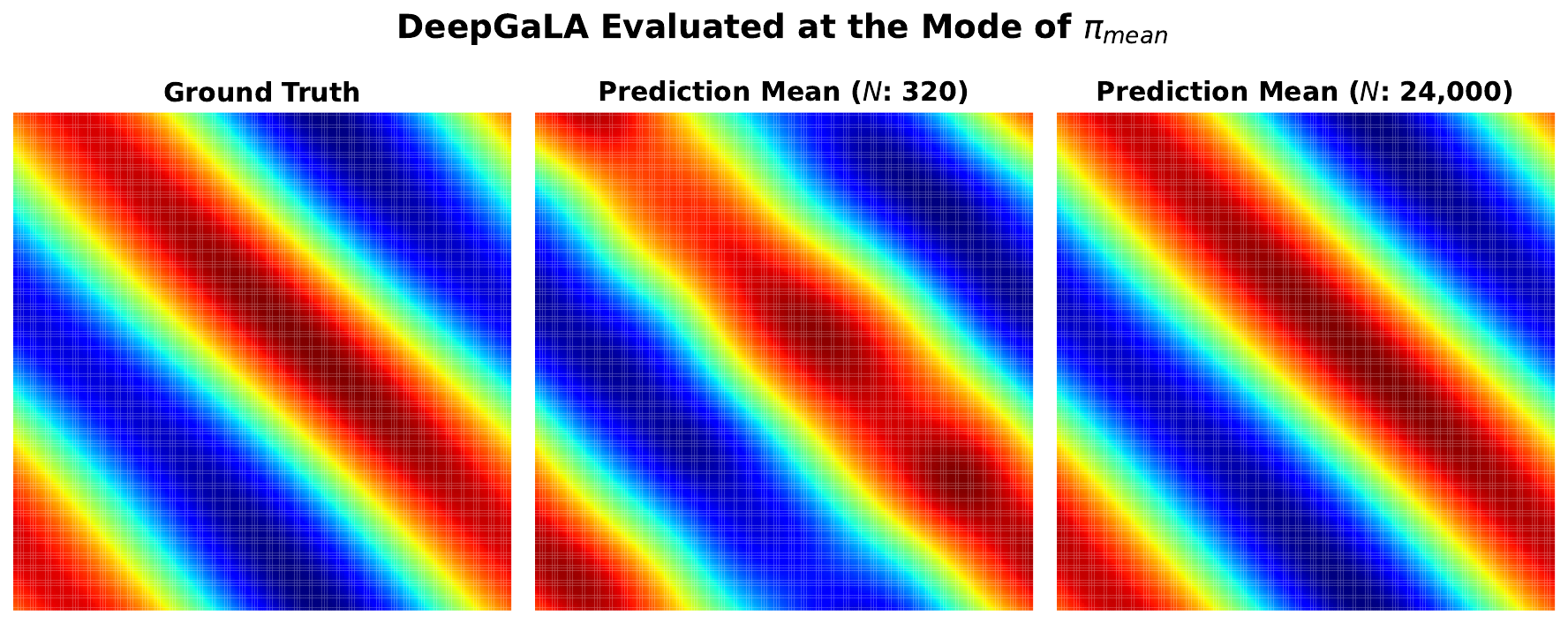}
\caption{NV: Prediction of the Deep GaLA evaluated at the mode of the $\pi_{\mathrm{mean}}$. 
}
\label{nv_fig:3}
\end{figure}

\begin{figure}[ht]
\centering
\includegraphics[width=0.8\textwidth]{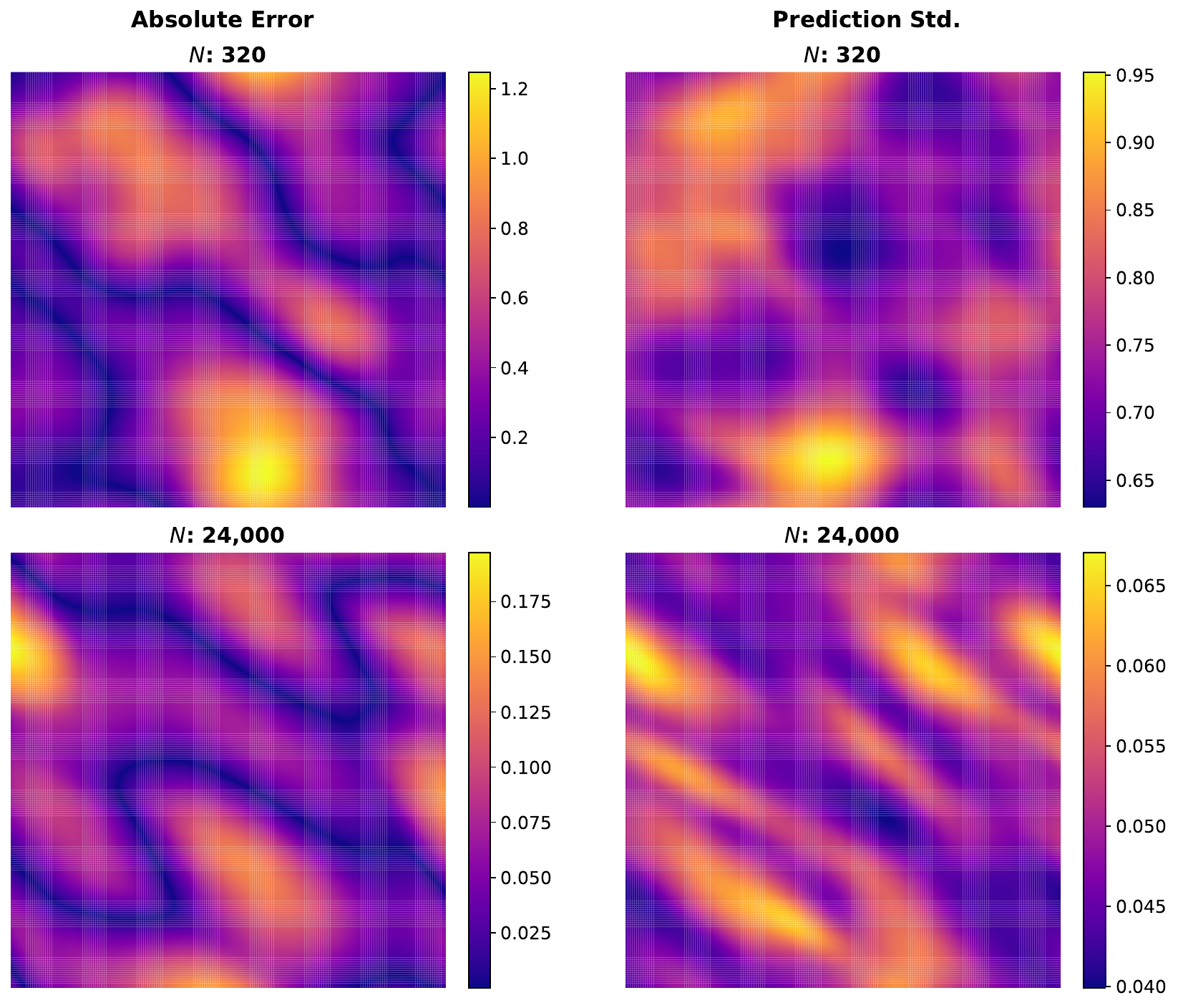}
\caption{NV: Absolute error and standard deviation of the DeepGaLA predictions in the low and high data regimes.}
\label{nv_fig:4}
\end{figure}

\section{Discussion and Conclusion}\label{sec:6}

Inverse problems in differential equations have significant computational challenges, particularly in a Bayesian setting where millions of evaluations of the forward model are required for the evaluation of the posterior. To mitigate this cost, we investigated the use of neural networks as surrogate models. This work provides two main contributions. First, we introduced a neural surrogate equipped with a Laplace approximation, called DeepGaLA, which offers an efficient way to approximate the forward model through a random neural surrogate. Second, we proposed using Delayed-Acceptance MCMC as a practical tool to assess the accuracy of any surrogate posterior approximations when the true posterior is unavailable or prohibitively expensive to compute.

DeepGaLA enables to approximate the posterior distribution through either the mean or marginal approximation. Our numerical experiments show that the mean approximation achieves accuracy comparable to the PIGP surrogate, and in some cases, DeepGaLA offers faster likelihood evaluations. In low-data regimes, the marginal approximation becomes particularly valuable, since it appropriately reflects the error in the
approximation of the neural network, thus avoiding being overconfident around the wrong solution. We have also demonstrated that using DA–MCMC as a posterior validator provides a rigorous and effective way to assess surrogate accuracy. This approach is especially useful when the ground-truth posterior is unavailable or computationally expensive. The empirical results agree with the theoretical properties of the validation metric derived in our framework.

A key advantage of neural-network surrogates over PIGPs is their ability to scale to nonlinear PDEs, whereas the PIGP framework is inherently limited to linear operators. Moreover, as the dimensionality of the parameter space increases, PIGP requires more training data to achieve a given level of accuracy. In such settings, obtaining  additional data can become challenging, and the evaluation speed is also adversely affected. In contrast, the evaluation speed of DeepGaLA depends primarily on the network architecture and is largely independent of both the size of the training dataset and the dimensionality of the parameter space.

\section*{Acknowledgements}
This work has made use of the resources provided by the Edinburgh Compute and Data Facility (ECDF) (http://www.ecdf.ed.ac.uk/).
ALT was partially supported by EPSRC grants EP/X01259X/1 and EP/Y028783/1. 
AV is supported by the "UNREAL: Unified Reasoning Layer for Trustworthy ML" project (EP/Y023838/1) selected by the ERC and funded by UKRI EPSRC.

\appendix

\section{Parameter optimization}
\label{appendix:parameters}

Both surrogates studied in this work, Physics Informed Gaussian process and DeepGaLA, have hyperparameters that must be selected carefully. On the one hand, the hyperparameters of a Gaussian process depend on the kernel used. For instance, the Matern 5/2 and Gaussian kernels have the parameters $\sigma_{dG}^{2}$ and $l$, which control the magnitude of the covariance and the length-scale at which the entries of the kernel are correlated. On the other hand, from Equation \eqref{hessian_eq}, we see that DeepGaLA has two hyperparameters, $\gamma$ and $\sigma_{dG}^{2}$, which encode prior beliefs over the final layer parameters and the noise of the fictional data. To determine these hyperparameters, we use the marginal log-likelihood in both cases. For the Gaussian process, we have:

\begin{equation*}
Z_{\text{GP}} = -\frac{1}{2} u(\Theta)^{T} K(\Theta, \Theta)^{-1} u(\Theta) 
- \frac{1}{2} \log |K(\Theta, \Theta)| 
- \frac{N}{2} \log (2\pi),
\end{equation*}

where $u(\Theta)$ is the vector of solutions in the training set $\Theta$. For DeepGaLA, we have:

\begin{equation*}
Z_{f_{W}^{a}} = -\frac{1}{2 \sigma_{dG}^{2}} \mathcal{L}(\mathcal{D}; \mathrm{W}_{\text{MAP}})-\frac{1}{2 } \mathrm{W}_{\text{MAP}}^{T} \gamma^{-2}  \mathrm{W}_{\text{MAP}}
- \frac{N}{2} \log (2\pi \sigma_{dG}^{2}) 
- \frac{1}{2} (\log |\Lambda| -\log |\gamma^{-2}|).
\end{equation*}

Following the work of \cite{Rasmussen2004, immer2021scalable}, we maximize the marginal likelihood:

\begin{equation*}
\arg \max_{\Xi} Z_{S},
\end{equation*}
where $\Xi$ is the set of hyperparameters to optimize and $Z_{S} = \{ Z_{\text{GP}}, Z_{f_{W}^{a}} \}$.

\section{Cost of Evaluation Analysis}\label{appendix:costeval}

Consider the architecture of the neural network $f_{a}^\mathrm{ENC}(\cdot;\mathrm{W})$ given by Equation \eqref{eq:nn_composition} and \eqref{eq:modiact}, and assume that the cost of evaluating $\sigma(\cdot)$ is one unit. Then, the cost of evaluating the encoders $V$ and $U$ is $2(d_n d_0 + d_n).$ The cost of evaluating Equation \eqref{eq:modiact} can be derived as follows. First, the cost of computing $\sigma \circ C_l(y)$ is $2(d_n d_l + d_n), \quad \text{for } l = 0, \dots, L-1.$ Next, the term $(\sigma \circ C_l(y) )\odot U(z)$ has a cost of $d_n$, while $(1 - \sigma \odot C_l(y)) \odot V(z)$ has a cost of $2d_n$. Therefore, the input layer has a cost of $2 d_n d_0 + 5 d_n$, while each hidden layer has a cost of $5d_n + 2 d_n^2.$ The final layer has a cost of $2 d_{L}d_n + d_L$. Taking into account that the surrogate approximates the problem in question in $d_{y}$ points. Then, the total cost of evaluating the mean of DeepGaLA is $d_{y}( 2 d_{L}d_n + d_L+ (L - 1)(5 d_n + 2 d_n^2) + 2 d_n  d_{0} + 5d_n + 4(d_n d_{0} + d_n)) =  d_{y}(2 d_{L}d_n + d_L + (L-1)(5d_{n} + 2d_{n}^{2}) + d_{n}(6 d_{0}+ 9)).$

Computing the variance of DeepGaLA requires evaluating the activations of the last hidden layer, which in turn requires evaluating the entire network. This can be efficiently obtained through a forward pass in PyTorch. Therefore, we assume that the cost consists of the cost of evaluating the mean plus an additional cost of $2 d_n^2 + 2 d_n$ for each $\Phi^{(L-1)}(z_{i};\mathrm{W}_{\text{MAP}})^{T}\Lambda \Phi^{(L-1)}(z_{j};\mathrm{W}_{\text{MAP}})$ for $i,j=1,...,d_{y}$. Therefore, the total cost of prediction with DeepGaLA, including uncertainty quantification, is $d_{y}(2 d_{L}d_n + d_L + (L-1)(5d_{n} + 2d_{n}^{2}) + d_{n}(6 d_{0}+ 9)) + 2d_{y}^{2}(d_n^2 + d_n)$. 

On the other hand, consider the cost of evaluating the PIGP, assuming $N$ training solutions. We assume that the cost of evaluating the kernel at a test point scales as $\mathcal{O}(N)$. Under this assumption, the cost of computing the predictive mean is $2 d_y (d_y + d_f + d_g) N + 3N.$  To compute the predictive variance of the PIGP, one must evaluate the second equation in \eqref{eq:pigp_pred}, which involves computing $K(\Theta, \Theta)^{-1}
\begin{bmatrix}
K_{uu}(\theta', \Theta) \\[2pt]
K_{ug}(\theta', \Theta) \\[2pt]
K_{uf}(\theta', \Theta)
\end{bmatrix}$, this to ensure numerical stability.
This operation incurs a computational cost of $\mathcal{O}\big((d_y + d_f + d_g)^3 N^3\big).$ Then the resulting matrix is multiplied by a matrix of size
$d_y \times N(d_y + d_g + d_f)$, which has a cost of
$2 d_y^2 N(d_y + d_g + d_f)$. The cost of forming the last matrix is $d_y^2$, and the subtraction of the matrices adds another $d_y^2$. Therefore, the cost of computing the variance is
$(d_y + d_f + d_g)^3 N^3
+ 2 d_y^2 N(d_y + d_g + d_f)
+ 2 d_y^2.$

\newpage
\bibliographystyle{abbrv}
\bibliography{references}  






\end{document}